\documentclass{article}

\PassOptionsToPackage{numbers, compress}{natbib}

\usepackage[main,final]{neurips_2026}
\setcitestyle{square,numbers,comma}

\usepackage[utf8]{inputenc} %
\usepackage[T1]{fontenc}    %
\usepackage{hyperref}       %
\usepackage{url}            %
\usepackage{booktabs}       %
\usepackage{multirow}       %
\usepackage{amsmath}        %
\usepackage{amsfonts}       %
\usepackage{amsthm}         %
\usepackage{nicefrac}       %
\usepackage{microtype}      %
\usepackage[table]{xcolor}  %
\usepackage{titletoc}       %
\usepackage{graphicx}       %
\usepackage{subcaption}     %
\usepackage{wrapfig}        %
\usepackage{algorithm}      %
\usepackage{algpseudocode}  %

\definecolor{org}{HTML}{F8A145}
\hypersetup{
  colorlinks=true,
  linkcolor=blue,
  urlcolor=org,
  citecolor=org,
  linktoc=all
}

\newtheorem{theorem}{Theorem}
\newtheorem{lemma}[theorem]{Lemma}
\newtheorem{remark}[theorem]{Remark}

\newcommand{\Lcls}{\mathcal{L}_{\mathrm{cls}}}

\title{CODI: Improving Self-Supervised Vision Transformers with Cross Distillation}

\author{\parbox{13cm}{\centering
  {\large\bfseries
    Siran Dai$^{1,2}$ \quad Qianqian Xu$^{3,4}$\thanks{Corresponding authors: Qianqian Xu and Qingming Huang.} \quad Peisong Wen$^{5}$ \\
    Yang Liu$^{5}$ \quad Xiaochun Cao$^{6}$ \quad Qingming Huang$^{5,3*}$}\\
  {\normalsize\normalfont
    $^1$ Institute of Information Engineering, Chinese Academy of Sciences \\[-1pt]
    $^2$ School of Cyber Security, University of Chinese Academy of Sciences \\[-1pt]
    $^3$ State Key Laboratory of AI Safety, Institute of Computing Technology, CAS \\[-1pt]
    $^4$ Beijing Academy of Artificial Intelligence \\[-1pt]
    $^5$ School of Computer Science and Tech., University of Chinese Academy of Sciences \\[-1pt]
    $^6$ School of Cyber Science and Technology, Shenzhen Campus, Sun Yat-sen University \\[-1pt]
  }
  {\normalfont\ttfamily\small daisiran@iie.ac.cn \enspace xuqianqian@ict.ac.cn \enspace wenpeisong@ucas.ac.cn \\[-1pt]
    liuyang232@mails.ucas.ac.cn \enspace caoxiaochun@mail.sysu.edu.cn \enspace qmhuang@ucas.ac.cn}
}}

\begin{document}

\maketitle
\vspace{-1.0em}

\begin{abstract}
  Self-supervised Vision Transformers learn strong image-level representations, but their patch-level representations can be unreliable for dense prediction. We study this mismatch and find that the class token concentrates global semantics on a small number of representative patches, while dense prediction errors often occur on low-attention regions. We provide a local gradient analysis showing that, along the \texttt{[CLS]} attention readout, image-level self-distillation gradients are weighted by \texttt{[CLS]}-to-patch attention mass. To provide masked patches with global semantics without changing the architecture, we propose Cross Distillation (CODI), which mixes a small amount of the teacher class-token distribution into each teacher patch-token target. CODI gives every masked patch an attention-independent semantic correction while preserving the standard self-distillation architecture. Experiments on dense and image-level benchmarks show that CODI improves both dense and image-level performance under aligned pretraining settings. Code is available at \url{https://github.com/EldercatSAM/CODI}.
\end{abstract}
\vspace{-1.2em}

\section{Introduction}

\begin{wrapfigure}[18]{r}{0.45\textwidth}
    \vspace{-1.5em}
    \centering
    \includegraphics[width=\linewidth,height=0.88\linewidth]{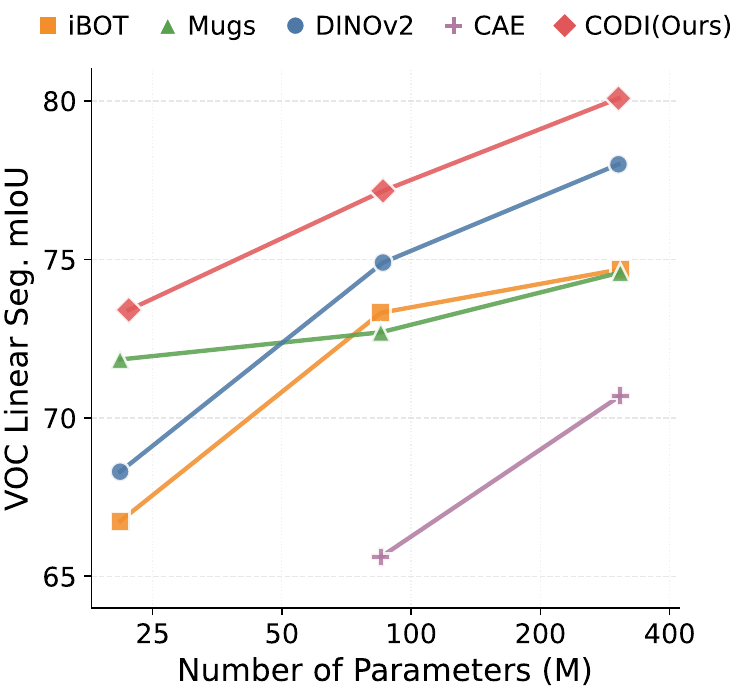}
    \caption{Comparison of CODI with baselines on Pascal VOC segmentation.}
    \label{fig:scaling}
    \vspace{-0.8em}
\end{wrapfigure}

A general-purpose self-supervised Vision Transformer (ViT) should support both image-level recognition and dense prediction. These two goals rely on different representations: image-level recognition mainly uses the class token, whereas dense prediction depends directly on patch tokens. Therefore, ideal approaches should not only learn global semantics in the class token but also pass these semantics to patch tokens. Recent teacher-student self-distillation methods \citep{dino, ibot, dinov2} have achieved great success by training a student with teacher-provided targets under both image-level discriminative and masked patch-level reconstruction losses. Despite this success, recent works \citep{sdd, dinov3} observe that this semantic delivery is imperfect: during DINOv2-style pretraining, dense transfer performance can degrade in the later stage, indicating that patch-level representations do not always improve together with global semantic tokens.

Since image-level self-distillation is optimized through the \texttt{[CLS]} token, we ask: \textit{why does a stronger \texttt{[CLS]} representation not automatically lead to stronger patch-level representations?} To diagnose this mismatch, we examine how the class-token readout couples to patch tokens during DINOv2-style pretraining. As shown in Fig.~\ref{fig:motivation}, the \texttt{[CLS]}-to-patch attention mass decreases in the later stage of training and follows a trend similar to dense transfer performance. The attention pattern is also highly selective: the \texttt{[CLS]} token repeatedly focuses on a small number of visually discriminative regions, such as the eyes or mouths of animals. We call these high-\texttt{[CLS]}-attention regions representative patches, meaning that they serve as representative evidence for the image-level readout.

This selectivity is natural for image-level recognition, where a few discriminative patches can be sufficient for a confident prediction. Dense prediction, however, evaluates patch tokens at every spatial location. A patch that is not selected by the global readout can still be important for segmentation or depth estimation. Therefore, the relevant issue is not whether the model attends to an entire object, but whether semantic information is reliably available beyond the sparse evidence used by the \texttt{[CLS]} token. Consistent with this view, we observe that downstream errors are concentrated on low-attention, non-representative patches.

Why does self-distillation produce such an uneven allocation? We identify an attention-gated semantic path that helps explain this behavior. Along the \texttt{[CLS]} readout in ViT, image-level self-distillation sends semantic gradients to patch tokens in proportion to their \texttt{[CLS]}-to-patch attention mass. Through this path, high-attention representative patches receive stronger semantic updates, while low-attention patches receive weaker updates. This bias is well aligned with image-level recognition, but can leave many patch tokens less reliable for dense prediction.

\begin{figure}[t]
    \centering
    \includegraphics[width=\linewidth]{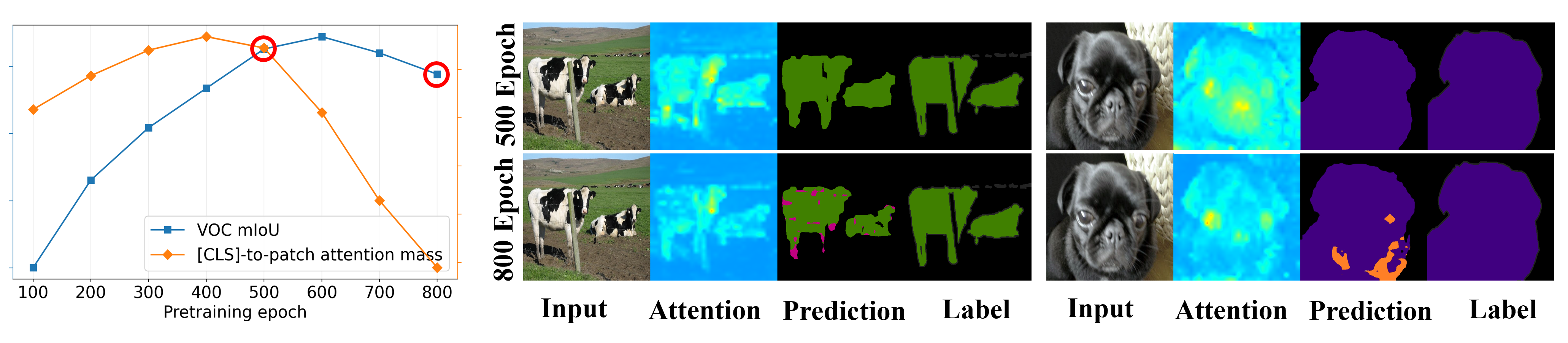}
    \caption{\textbf{Left:} During DINOv2 pretraining, downstream dense performance and \texttt{[CLS]}-to-patch attention mass follow similar trends. \textbf{Right:} The \texttt{[CLS]} token often focuses on representative patches, while prediction errors are more likely to occur on non-representative patches.}
    \label{fig:motivation}
    \vspace{-10pt}
\end{figure}

This analysis suggests that the teacher already contains useful global semantic information, but the standard training objective delivers this information to patch tokens unevenly through the \texttt{[CLS]} attention path. Based on this insight, we propose Cross Distillation (CODI), a target-side modification to masked patch distillation. For each masked patch, CODI mixes a small amount of the teacher class-token distribution into the teacher patch-token target. The student still learns from a primarily local patch target, but each masked patch also receives a weak global semantic correction that is not explicitly gated by its \texttt{[CLS]} attention mass. In this way, CODI improves the semantic reliability of low-attention patches while preserving the spatial specificity of patch supervision.

Empirically, we show that across seven evaluation settings spanning semantic segmentation, depth estimation, video object segmentation, and image classification, CODI \textit{outperforms} its DINOv2 \cite{dinov2} counterpart when the pretraining setting is aligned. As shown in Fig.~\ref{fig:scaling}, CODI's performance scales with model size. These results reveal the potential of CODI as a more efficient training paradigm.

\section{Related Work}

\subsection{Self-supervised Vision Transformers}

Self-supervised Vision Transformers~\citep{vit,pyramid_ViT,dino} mainly follow two paradigms:
\textbf{1) Contrastive methods} use view discrimination or cross-view consistency~\citep{simclrv2,mocov3,dino}.
\textbf{2) Prediction-based methods} predict masked content in pixel~\citep{mae,SimMIM,cae} or latent space~\citep{ijepa,ibot,dinov2}.
Among these methods, iBOT~\citep{ibot} uses an online tokenizer to provide self-distillation targets for both class tokens and masked patch tokens, while DINOv2~\citep{dinov2} scales this framework using stronger training strategies and larger-scale data.

\subsection{Self-supervised Dense Representation Learning}

Despite progress in image-level SSL~\citep{SimCLRv1,MoCov1,byol,swav,barlowtwins}, dense prediction requires spatially meaningful local semantics~\citep{dense_intro_2,xu2020noisy,yun2022patch,wang2022cp,wang2023cut,salehi2023time,seong2023leveraging,niu2024unsupervised,wen2025semanticconcentrationselfsuperviseddense}. Approaches include local feature alignment~\citep{densecl,liu2026staticdynamicexploringselfsupervised}, object-part grouping~\citep{leopart,wen2026slotjepaobjectstructuredpredictioninterfaces}, and masked prediction~\citep{mae,VJEPA2_1,liu2025futurepasttamingtemporal}.

Longer self-supervised training can also degrade dense representations. DINOv3~\citep{dinov3} restores dense structure through Gram anchoring, while DSE~\citep{sdd} provides structural measures for model selection and regularization. CODI instead modifies patch targets during pretraining to address the representative-patch issue without an additional fine-tuning stage. Section~\ref{subsec:gram_anchoring} compares CODI with Gram anchoring under aligned settings.

\section{The Representative Patch Phenomenon}
\label{sec:patch-starvation}

\subsection{Empirical Observation}
\label{subsec:observation}
To understand the root cause of the self-supervised dense degradation (SDD) phenomenon, we first empirically study the failure modes of existing baselines. Specifically, we train DINOv2 \citep{dinov2} on ImageNet-1K \citep{krizhevsky2012imagenet} from scratch for 800 epochs, and evaluate the final checkpoint on Pascal VOC \citep{voc}. We adopt the linear semantic segmentation setting, where the pretrained model is frozen and a linear head is trained to make pixel-level predictions based on patch representations.

Fig.~\ref{fig:motivation} visualizes the \texttt{[CLS]}-to-patch attention maps and the predicted pixel labels. We observe two patterns: 1) Models tend to assign high attention to several key patterns, such as the eyes and mouth of a cat, rather than to the whole object. We refer to these key patterns as \texttt{Representative Patches}; in quantitative analyses, we define them as the top-$r\%$ image patches ranked by the final-layer \texttt{[CLS]}-to-patch attention mass averaged over heads, excluding register tokens. Similar discriminative regions frequently receive high attention across different instances. 2) Prediction errors usually occur on non-representative patches, such as the wheels of a bus.
To determine whether this phenomenon arises from an observation error or a systematic bias, we bin image patches by normalized \texttt{[CLS]}-to-patch attention mass and compute the segmentation error rate within each bin. As shown in Fig.~\ref{fig:representative_patch_analysis}, we observe a monotonic trend between the error rate and the received global attention score. Since dense prediction tasks, such as semantic segmentation, require every patch, not only the representative patches, to be semantically meaningful, this observation motivates us to further study the cause of this phenomenon and address it.

\subsection{Theoretical Justification}
\label{subsec:theory}

The observation above raises a natural question: if DINO/iBOT already includes image-level self-distillation, why do some object patches still become unreliable for dense prediction? Our explanation is that the strong global target is optimized through the \texttt{[CLS]} token. For a patch to benefit from this target through the \texttt{[CLS]} branch, it must first be selected by \texttt{[CLS]} attention. Therefore, the same softmax mechanism that helps the model identify representative evidence can also make the global target weak for low-attention patches.

\begin{figure}[t]
    \centering
    \begin{tabular}{@{}c@{\hspace{0.015\textwidth}}c@{\hspace{0.015\textwidth}}c@{\hspace{0.015\textwidth}}c@{}}
        \begin{minipage}[t]{0.235\textwidth}
            \centering
            \includegraphics[width=\linewidth]{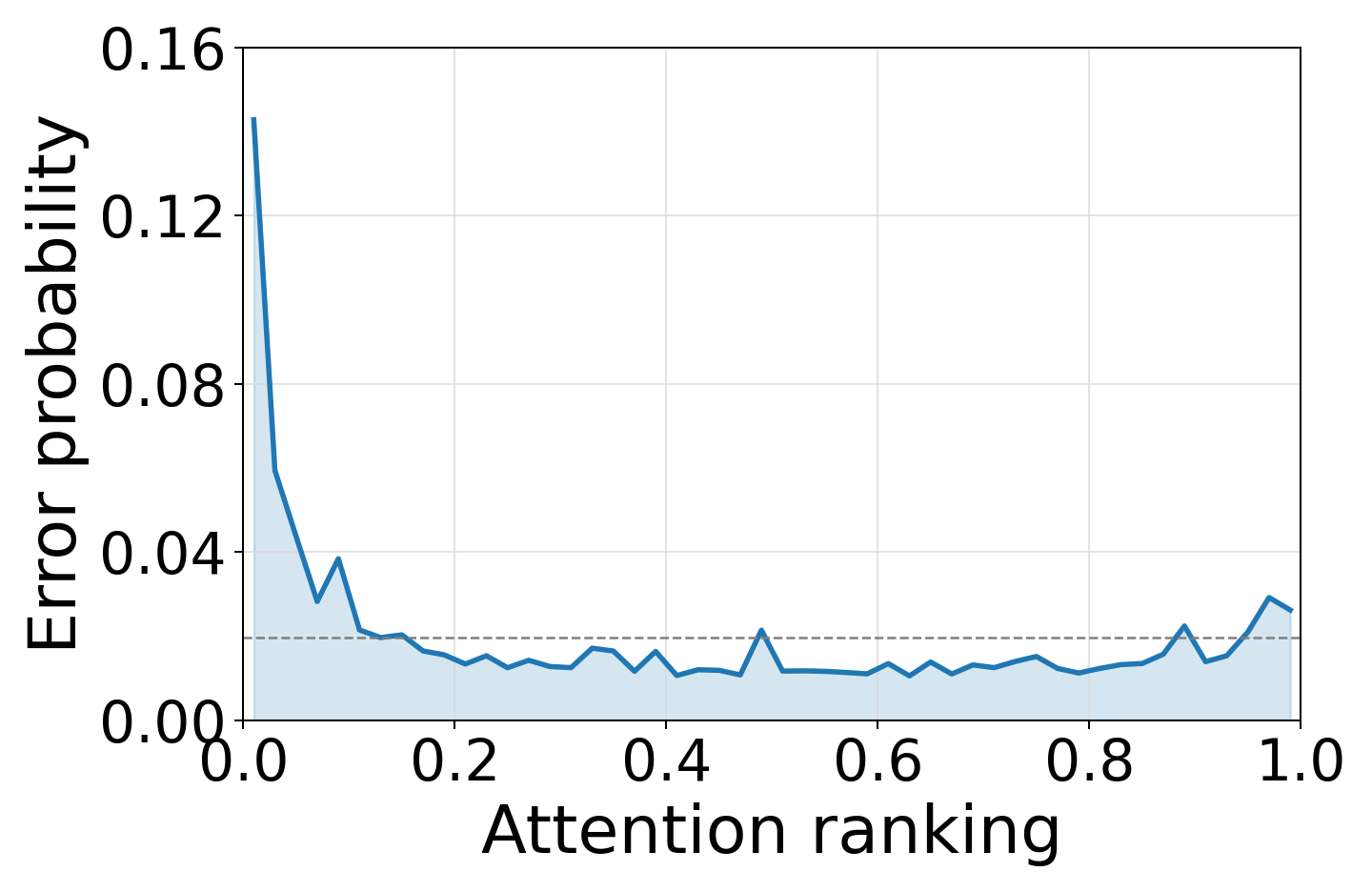}\\[-0.3em]
            {(a) Bird}
        \end{minipage}
        &
        \begin{minipage}[t]{0.235\textwidth}
            \centering
            \includegraphics[width=\linewidth]{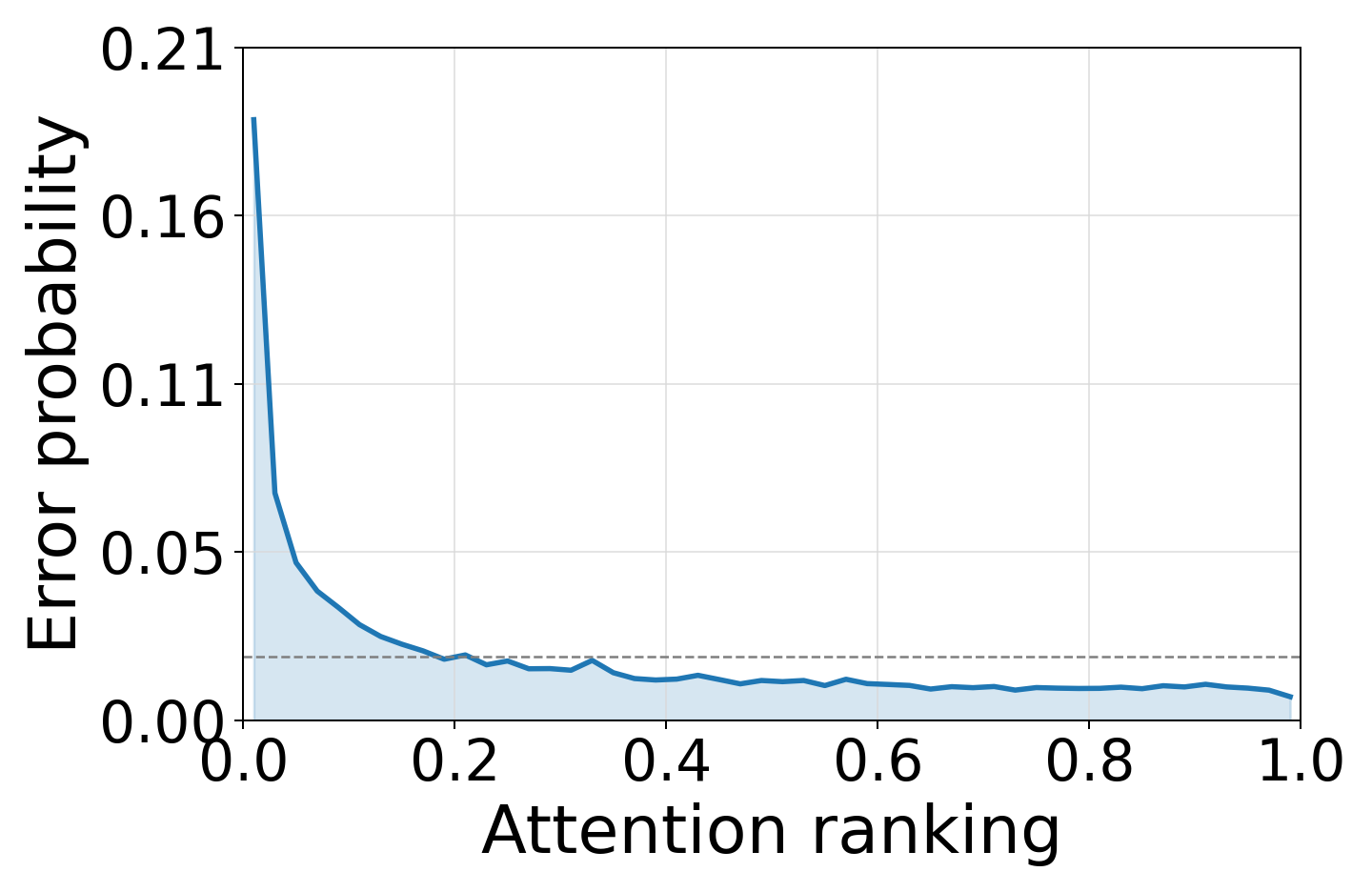}\\[-0.3em]
            {(b) Bus}
        \end{minipage}
        &
        \begin{minipage}[t]{0.235\textwidth}
            \centering
            \includegraphics[width=\linewidth]{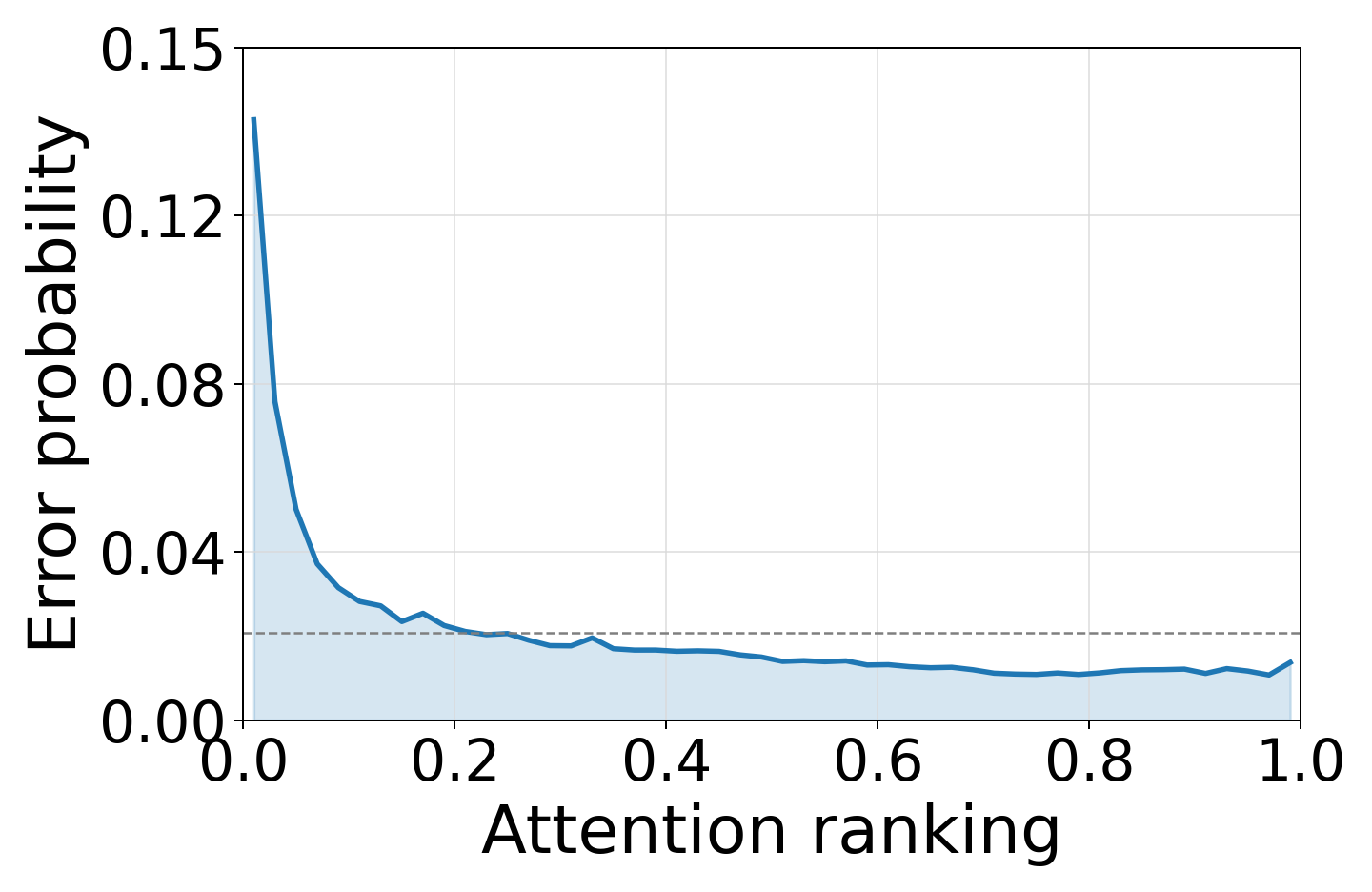}\\[-0.3em]
            {(c) Dog}
        \end{minipage}
        &
        \begin{minipage}[t]{0.235\textwidth}
            \centering
            \includegraphics[width=\linewidth]{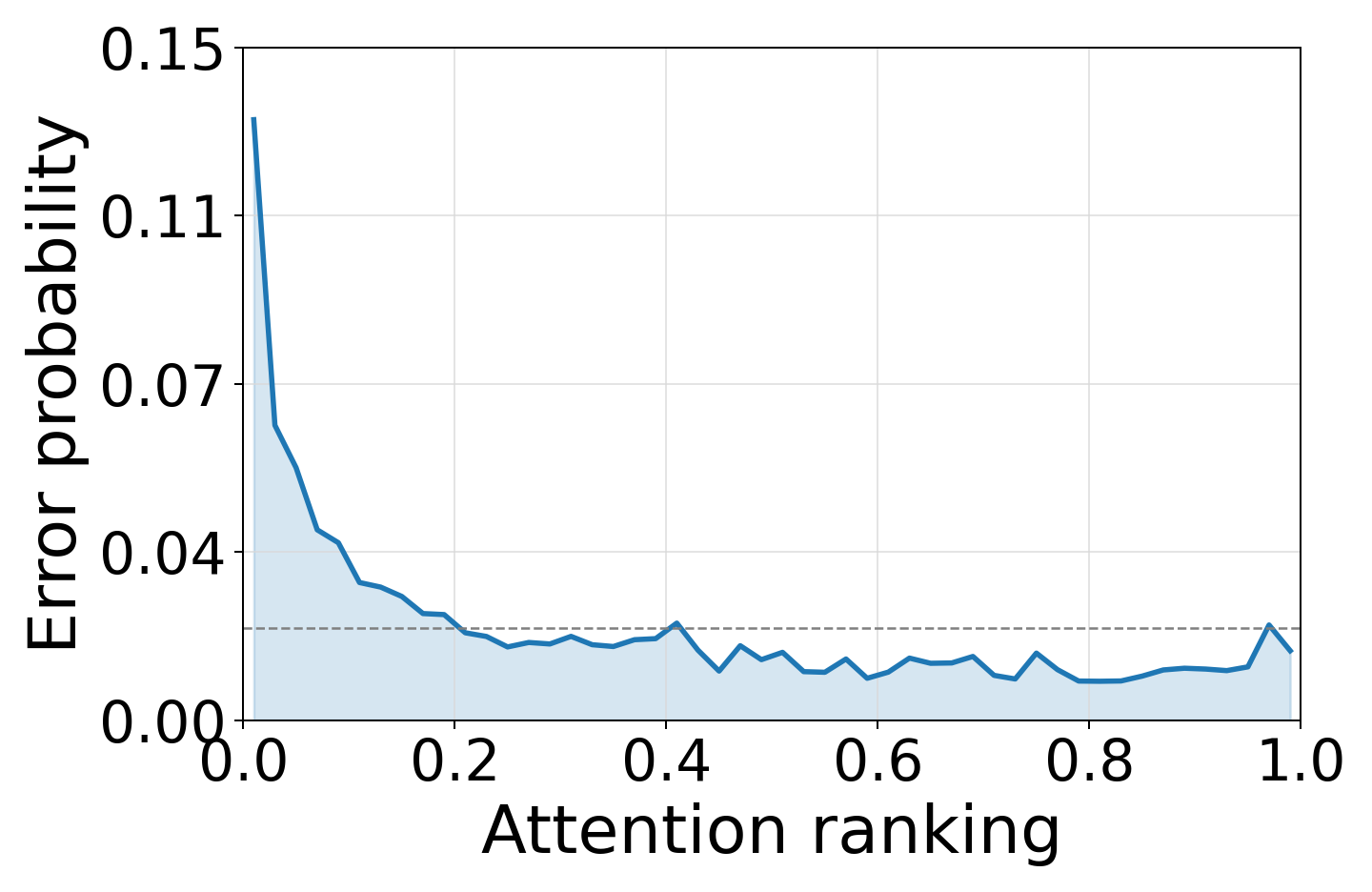}\\[-0.3em]
            {(d) Train}
        \end{minipage}
    \end{tabular}
    \caption{The model tends to make prediction errors on patches with low attention scores.}
    \label{fig:representative_patch_analysis}
\end{figure}

\paragraph{Local \texttt{[CLS]} attention model.}
This subsection analyzes a one-step local \texttt{[CLS]} readout under the image-level class-token loss only. The masked patch loss is not part of this local formulation; it may provide additional patch-token gradients, while our claim concerns the semantic gradient transmitted through the \texttt{[CLS]} attention path. Let \(P\) denote image patch tokens and let \(T\) denote all tokens competing in the final \texttt{[CLS]} attention row, including register tokens when used. For token \(i\in T\), let \(k_i\) and \(v_i\) be its key and value. The local \texttt{[CLS]} readout is
\begin{equation}
    s_i=\frac{a_{\mathrm{cls}}^\top k_i}{\sqrt d},
    \qquad
    \pi_i=\frac{\exp(s_i)}{\sum_{j\in T}\exp(s_j)},
    \qquad
    z=\sum_{j\in T}\pi_j v_j.
    \label{eq:local-readout}
\end{equation}
For patch-level analysis we consider \(i\in P\). We define the total \texttt{[CLS]}-to-patch mass as
\[
    \Pi_P=\sum_{i\in P}\pi_i,
\]
and rank patches within an image by the normalized patch attention \(\bar{\pi}_i=\pi_i/\Pi_P\). This abstraction treats upstream representations and other network paths as fixed, and studies the exact gradient allocation along this softmax-gated attention path.

We treat the teacher class-token target \(q_{\mathrm{cls}}^t\) as fixed in this local calculation, and define
\begin{equation}
    \mathcal{L}_{\mathrm{cls}}
    =
    \mathrm{CE}\left(q_{\mathrm{cls}}^t,\mathrm{softmax}(\phi_s(z)/\tau_s)\right),
    \qquad
    g:=\nabla_z \mathcal{L}_{\mathrm{cls}}.
    \label{eq:dino-loss}
\end{equation}
The construction of \(q_{\mathrm{cls}}^t\) is given in Sec.~\ref{sec:self_distill}.

\begin{theorem}[\texttt{[CLS]} attention gates semantic gradients along the value and score paths]
\label{thm:cls-gating}
For any patch token \(i\in P\), the local readout in Eq.~\eqref{eq:local-readout} sends a value-path gradient and a score update of the form
\begin{equation}
    \nabla_{v_i}\Lcls=\pi_i g,
    \qquad
    -\frac{\partial\Lcls}{\partial s_i}=\pi_i u_i,
    \qquad
    u_i:=-\langle g,v_i-z\rangle,
    \label{eq:attention-gradient-main}
\end{equation}
where \(g\) denotes the global semantic gradient back-propagated from the teacher \texttt{[CLS]} target.
\end{theorem}

\begin{remark}
The scalar \(u_i\) measures the local semantic utility of patch \(i\): if \(u_i>0\), increasing \(s_i\) decreases the loss locally. The value-path gradient norm is exactly \(\|\nabla_{v_i}\mathcal{L}_{\mathrm{cls}}\|=\pi_i\|g\|\), and the score update also contains the multiplicative factor \(\pi_i\). Thus, through this \texttt{[CLS]} readout path, low-attention patches are weakly coupled to the image-level semantic target.
\end{remark}

\begin{wrapfigure}[14]{r}{0.4\textwidth}
    \vspace{-11pt}
    \centering
    \includegraphics[width=\linewidth]{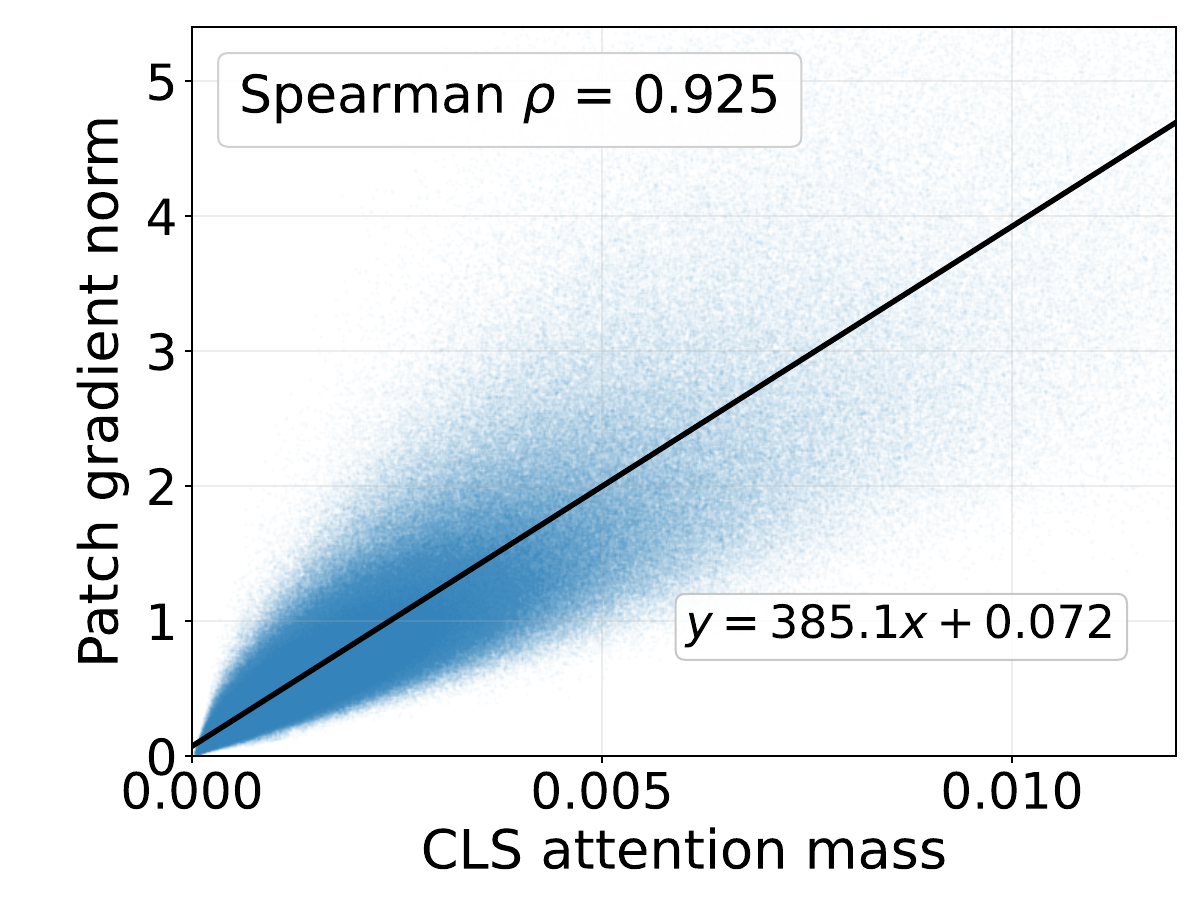}
    \vspace{-1.0em}
    \caption{\texttt{[CLS]} attention mass vs. DINO gradient norm.}
    \label{fig:cls_attn_grad_correlation}
    \vspace{-0.8em}
\end{wrapfigure}

Theorem~\ref{thm:cls-gating} identifies one mechanism by which representative patches can dominate the image-level semantic pathway. Classification benefits from sparse, discriminative evidence; dense prediction instead requires semantic features on many non-representative object regions. Consequently, longer self-distillation training can strengthen the \texttt{[CLS]} representation while leaving low-attention patches with weaker semantic guidance through this path, which matches our observation that segmentation errors concentrate on non-representative patches. Formal statements showing attention-mass amplification and sparse sufficiency of image-level readouts are provided in Appendix~\ref{app:attention-amplification} and Appendix~\ref{app:sparse-sufficiency}.
To validate whether our theory aligns with practical observations, which are influenced by the projection head, multi-head attention, and registers, we conduct an experiment to measure the correlation between \texttt{[CLS]}-to-patch attention mass and the gradient norm received from the DINO loss. As shown in Fig.~\ref{fig:cls_attn_grad_correlation}, these two quantities show a strong Spearman correlation of 0.925.

\section{Methodology}
\label{sec:method}
In this section, we introduce Cross Distillation (CODI) as a target-side modification motivated by the analysis above. Specifically, within a typical self-distillation framework, we mix a small proportion of teacher class-token semantics into the patch targets. The analysis shows that CODI adds a target-level semantic correction that is not explicitly weighted by \texttt{[CLS]} attention mass. An overview of CODI is presented in Fig.~\ref{fig:overall}.

\begin{figure}[t]
    \centering
    \includegraphics[width=\linewidth]{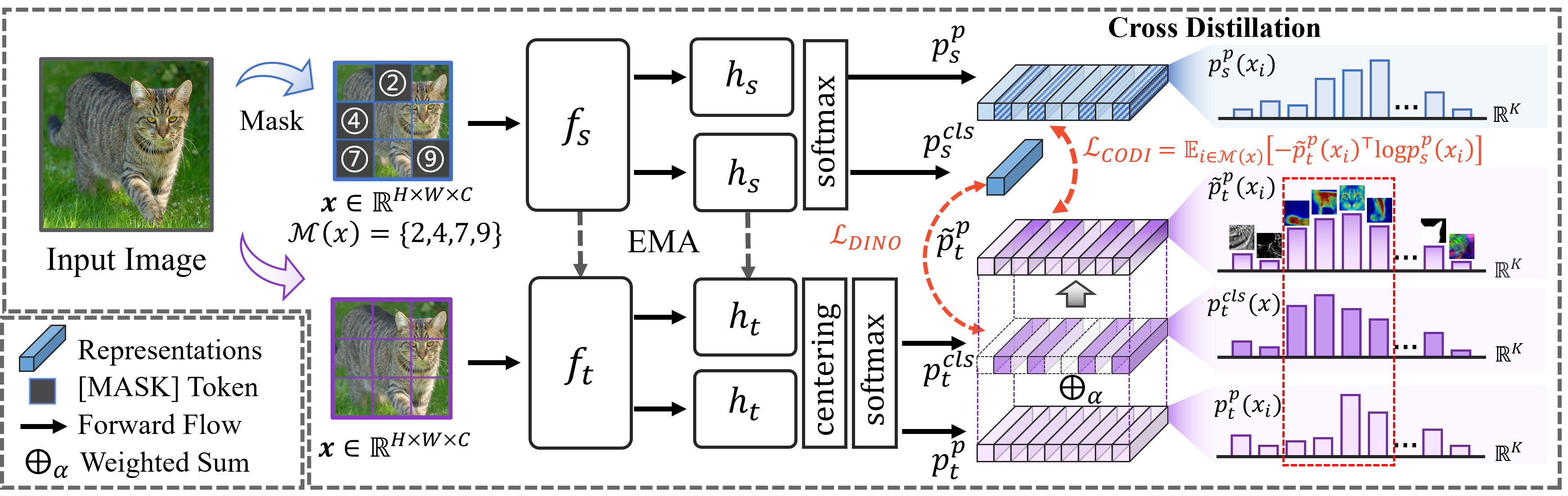}
    \caption{Overview of CODI\@. In the self-distillation framework, CODI constructs the masked-patch target by mixing the teacher patch-token distribution with a small teacher class-token component.}
    \label{fig:overall}
\end{figure}

\subsection{Self-Distillation Framework}
\label{sec:self_distill}

We build on a typical self-distillation framework \citep{dino, ibot, dinov2}. A student network $f_s$ is trained to match the target distributions produced by a momentum teacher network $f_t$, which is updated as an exponential moving average of the student. Both networks use a ViT backbone \citep{vit}. Given an augmented view $x$, the teacher produces a class-token embedding $z_t^{\mathrm{cls}}(x)$ and patch-token embeddings $\{z_{t,i}^{\mathrm{p}}(x)\}_{i=1}^{N}$, where $N$ is the number of patches. After a shared projection head, branch-specific centering, and sharpening, we obtain the teacher class-level and patch-level target distributions:
\begin{align}
p_t^{\mathrm{cls}}(x)
=
\mathrm{softmax}
\left(
\frac{
h_t(z_t^{\mathrm{cls}}(x)) - c^{\mathrm{cls}}
}{
\tau_t
}
\right), \quad 
p_t^{\mathrm{p}}(x_i)
=
\mathrm{softmax}
\left(
\frac{
h_t(z_{t,i}^{\mathrm{p}}(x)) - c^{\mathrm{p}}
}{
\tau_t
}
\right),
\end{align}
where $h_t$ is the shared teacher projection head that maps both class and patch tokens to $K$-dimensional prototype logits, $c^{\mathrm{cls}}$ and $c^{\mathrm{p}}$ are separate running centers for the class and patch teacher targets, and $\tau_t$ is the teacher temperature. Student predictions are obtained by applying softmax to student logits divided by the student temperature $\tau_s$; they follow the same shared-head design with the student head $h_s$, but do not use teacher centering.
The class-token distribution captures the global semantic assignment of the view, while the patch-token distribution provides local supervision for each patch. Our goal is to couple these two levels of supervision when constructing masked patch targets.

\subsection{Cross Distillation via Target Mixing}
\label{sec:cross_distill}

In masked patch self-distillation, each masked student patch is usually supervised by the corresponding teacher patch target alone. This target is locally informative, but it may emphasize low-level appearance details and does not explicitly encourage all patch representations to preserve discriminative semantics. We therefore introduce \emph{Cross Distillation} (CODI), a target-side mixing strategy that transfers global class-token semantics into masked patch supervision.

For each teacher-student patch-distillation pair, and for each masked patch position $i \in \mathcal{M}(x)$, we construct a mixed teacher target by combining the teacher patch distribution with the teacher class-token distribution from the same teacher view that provides the patch target:
\begin{equation}
\widetilde{p}_t^{\mathrm{p}}(x_i)
=
(1-\alpha)
p_t^{\mathrm{p}}(x_i)
+
\alpha
p_t^{\mathrm{cls}}(x),
\end{equation}
where $\alpha \in [0,1)$ controls the strength of global semantic guidance. When \(\alpha=0\), CODI reduces to the standard iBOT loss. 
Since class and patch tokens are projected by the same head, both targets are distributions over the same $K$ prototype coordinates, although they are centered with different running centers before the softmax. The centers shift the logits before softmax but do not change the prototype coordinate system; therefore both distributions remain defined over the same prototype indices, and their convex combination remains a valid teacher target. The CODI loss is defined as:
\begin{equation}
\mathcal{L}_{\mathrm{CODI}}
=
-
\frac{1}{|\mathcal{M}(x)|}
\sum_{i \in \mathcal{M}(x)}
\widetilde{p}_t^{\mathrm{p}}(x_i)^\top
\log p_s^{\mathrm{p}}(x_i),
\end{equation}
where $p_s^{\mathrm{p}}(x_i)$ is the student prediction for the masked patch.
For image-level supervision, we keep the standard DINO objective, which aligns student and teacher class tokens across augmented views:
\begin{equation}
\mathcal{L}_{\mathrm{DINO}}
=
-
\sum_{(u,v)\in\mathcal{V}}
p_t^{\mathrm{cls}}(v)^\top
\log p_s^{\mathrm{cls}}(u),
\end{equation}
where $\mathcal{V}$ denotes the set of teacher-student view pairs. The final objective is
\begin{equation}
\mathcal{L}
=
\mathcal{L}_{\mathrm{DINO}}
+
\lambda_1 \cdot
\mathcal{L}_{\mathrm{CODI}}
+
\lambda_2 \cdot
\mathcal{L}_{\mathrm{koleo}}.
\end{equation}

Here, $\mathcal{L}_{\mathrm{koleo}}$ denotes the KoLeo loss adopted from DINOv2 \citep{dinov2}. Intuitively, CODI modifies the label used in masked image modeling. Instead of asking each masked patch to reconstruct a purely patch-level teacher target, we encourage the masked patch prediction to contain global semantic information from the teacher class token. This weak global guidance helps the model avoid overfitting to local texture or appearance details during reconstruction, while preserving more discriminative image-level semantics.

\subsection{Theoretical Analysis}
\label{sec:theory}

The result in Sec.~\ref{subsec:theory} shows that the standard \texttt{[CLS]} path gates the global semantic signal to patch $i$ by the attention factor $\pi_i$. We now show that CODI adds a direct target-level semantic path whose strength is controlled by $\alpha$ and does not depend on $\pi_i$. For a masked patch $i$, write
\begin{equation}
    q_i^t:=p_t^{\mathrm{p}}(x_i),
    \qquad
    q_{\mathrm{cls}}^t:=p_t^{\mathrm{cls}}(x),
    \qquad
    \xi_i:=h_s(z_{s,i}^{\mathrm{p}})/\tau_s,
    \qquad
    p_i^s=\mathrm{softmax}(\xi_i),
\end{equation}
where $q_i^t$ and $q_{\mathrm{cls}}^t$ are the teacher targets produced by the same projection head with their respective centers, and $\xi_i$ is the temperature-scaled student patch logit. The original iBOT patch loss and the CODI patch loss are
\begin{equation}
    \mathcal{L}_{\mathrm{patch}}^0(i)=-(q_i^t)^\top\log p_i^s,
    \qquad
    \mathcal{L}_{\mathrm{patch}}^\alpha(i)=-(\widetilde q_i^t)^\top\log p_i^s,
    \qquad
    \widetilde q_i^t=(1-\alpha)q_i^t+\alpha q_{\mathrm{cls}}^t.
    \label{eq:mixed-target}
\end{equation}

\begin{theorem}[Target mixing gives attention-independent guidance]
\label{thm:codi-gradient}
Under negative logit-gradient flow for $\xi_i$,
\begin{equation}
    \dot\xi_i^0=q_i^t-p_i^s,
    \qquad
    \dot\xi_i^\alpha=q_i^t-p_i^s+
    \alpha(q_{\mathrm{cls}}^t-q_i^t).
    \label{eq:mixed-gradient-main}
\end{equation}
Thus, each masked patch prediction receives an additional target-level logit-gradient term
$\alpha(q_{\mathrm{cls}}^t-q_i^t)$, which is not multiplied by its \texttt{[CLS]} attention mass $\pi_i$.
\end{theorem}

\begin{remark}
    The mixed target moves the local patch target toward the global teacher distribution along a line segment in the prototype simplex. Unlike the \texttt{[CLS]} attention gradient, this correction is not explicitly gated by $\pi_i$.
\end{remark}

\section{Experiments}
We first describe the pretraining settings in Sec.~\ref{subsec:exp_setting}. Then, we evaluate CODI across seven protocols in Sec.~\ref{subsec:exp_overall} and compare with Gram anchoring in Sec.~\ref{subsec:gram_anchoring}. Finally, in Sec.~\ref{subsec:exp_understand}, we discuss the effect of cross distillation. Additional results on unsupervised segmentation, object detection, and smaller models are provided in Appendix~\ref{app:exps}.

\providecolor{ibotcyan}{cmyk}{.3,0,0,0}
\begin{table*}[t]
    \centering
    \caption{Quantitative results of linear semantic segmentation and classification. Gray rows denote oracles trained with substantially larger data/model; $\dagger$ denotes our reproduced result. The best and second-best results are highlighted with \textbf{bold} and \underline{underlined}, respectively.}
    \label{tab:segmentation_main}
    \label{tab:classification_main}
    \resizebox{\linewidth}{!}{
        \begin{tabular}{llcc|cc|cc|cc|cc|cc}
            \toprule
            \multirow{3}{*}{Methods} & \multirow{3}{*}{Arch.} & \multirow{3}{*}{Dataset} & \multirow{3}{*}{Epochs} & \multicolumn{6}{c|}{Semantic Segmentation} & \multicolumn{4}{c}{ImageNet-1K Classification} \\
            \cmidrule(lr){5-10} \cmidrule(lr){11-14}
            & & & & \multicolumn{2}{c|}{COCO} & \multicolumn{2}{c|}{ADE20K} & \multicolumn{2}{c|}{VOC} & \multicolumn{2}{c|}{$k$-NN} & \multicolumn{2}{c}{Linear} \\
            \cmidrule(lr){5-6} \cmidrule(lr){7-8} \cmidrule(lr){9-10} \cmidrule(lr){11-12} \cmidrule(lr){13-14}
            & & & & mIoU & Acc. & mIoU & Acc. & mIoU & Acc. & Top-1 & Top-5 & Top-1 & Top-5 \\
            \midrule
            \multicolumn{14}{l}{\textit{Distilled from Larger Models}} \\
            \rowcolor{gray!12}
            \textcolor{gray!65!black}{DINOv2 \citep{dinov2}} & \textcolor{gray!65!black}{ViT-S/14} & \textcolor{gray!65!black}{LVD-142M} & \textcolor{gray!65!black}{--} & \textcolor{gray!65!black}{57.4} & \textcolor{gray!65!black}{77.7} & \textcolor{gray!65!black}{43.5} & \textcolor{gray!65!black}{78.9} & \textcolor{gray!65!black}{79.4} & \textcolor{gray!65!black}{95.3} & \textcolor{gray!65!black}{79.0} & \textcolor{gray!65!black}{91.5} & \textcolor{gray!65!black}{80.4} & \textcolor{gray!65!black}{95.1} \\
            \rowcolor{gray!12}
            \textcolor{gray!65!black}{DINOv2 \citep{dinov2}} & \textcolor{gray!65!black}{ViT-B/14} & \textcolor{gray!65!black}{LVD-142M} & \textcolor{gray!65!black}{--} & \textcolor{gray!65!black}{59.1} & \textcolor{gray!65!black}{78.4} & \textcolor{gray!65!black}{46.3} & \textcolor{gray!65!black}{79.8} & \textcolor{gray!65!black}{81.3} & \textcolor{gray!65!black}{95.7} & \textcolor{gray!65!black}{82.0} & \textcolor{gray!65!black}{92.3} & \textcolor{gray!65!black}{83.9} & \textcolor{gray!65!black}{96.6} \\
            \rowcolor{gray!12}
            \textcolor{gray!65!black}{DINOv2 \citep{dinov2}} & \textcolor{gray!65!black}{ViT-L/14} & \textcolor{gray!65!black}{LVD-142M} & \textcolor{gray!65!black}{--} & \textcolor{gray!65!black}{58.2} & \textcolor{gray!65!black}{77.5} & \textcolor{gray!65!black}{46.5} & \textcolor{gray!65!black}{79.1} & \textcolor{gray!65!black}{80.4} & \textcolor{gray!65!black}{95.4} & \textcolor{gray!65!black}{83.8} & \textcolor{gray!65!black}{94.3} & \textcolor{gray!65!black}{85.9} & \textcolor{gray!65!black}{97.4} \\
            \midrule
            \multicolumn{14}{l}{\textit{Small Models}} \\
            MoCo v2 \citep{mocov2} & RN-50 & IN1K & 800 & 37.6 & 64.0 & 20.0 & 61.2 & 53.3 & 88.4 & 61.8 & 82.8 & 43.1 & 72.8 \\
            SwAV \citep{swav} & RN-50 & IN1K & 800 & 42.2 & 68.0 & 23.3 & 65.8 & 55.4 & 89.0 & 65.7 & 84.9 & 64.2 & 85.8 \\
            DINO \citep{dino} & RN-50 & IN1K & 800 & 40.7 & 66.6 & 22.4 & 64.5 & 53.1 & 88.4 & 67.5 & 86.0 & 65.9 & 87.2 \\
            MoCo v3 \citep{mocov3} & ViT-S/16 & IN1K & 300 & 43.4 & 70.6 & 24.4 & 70.5 & 61.7 & 91.4 & 32.2 & 50.5 & 54.8 & 78.8 \\
            DINO \citep{dino} & ViT-S/16 & IN1K & 800 & 39.5 & 67.1 & 23.1 & 67.4 & 50.0 & 88.4 & 74.4 & 90.0 & 76.0 & 92.7 \\
            iBOT \citep{ibot} & ViT-S/16 & IN1K & 800 & 47.7 & 72.7 & 28.5 & 72.3 & 66.7 & 92.4 & 75.1 & 89.1 & 77.0 & 93.0 \\
            Mugs \citep{mugs} & ViT-S/16 & IN1K & 800 & 48.7 & 72.6 & \textbf{30.9} & \underline{72.8} & \underline{71.8} & \underline{93.3} & 75.5 & 89.0 & 77.7 & 93.5 \\
            Leopart \citep{leopart} & ViT-S/16 & IN1K+COCO & 800+50 & \underline{48.8} & \underline{73.2} & 29.7 & 72.1 & 68.2 & 92.6 & 52.0 & 74.9 & 64.6 & 85.3 \\
            DINOv2$^\dagger$ \citep{dinov2} & ViT-S/14 & IN1K & 800 & 48.1 & 72.7 & 29.6 & 72.6 & 68.3 & 92.7 & \underline{76.9} & \underline{91.4} & \underline{77.9} & \underline{93.7} \\
            \rowcolor{ibotcyan!50}
            \textbf{CODI (ours)} & ViT-S/14 & IN1K & 800 & \textbf{50.3} & \textbf{74.0} & \underline{30.8} & \textbf{73.4} & \textbf{73.1} & \textbf{93.6} & \textbf{77.1} & \textbf{91.7} & \textbf{78.3} & \textbf{94.1} \\
            \midrule
            \multicolumn{14}{l}{\textit{Base Models}} \\
            CAE \citep{cae} & ViT-B/16 & IN1K & 800 & 45.5 & 71.0 & 27.0 & 71.9 & 64.9 & 92.2 & 37.7 & 60.9 & 21.0 & 42.0 \\
            CAE \citep{cae} & ViT-B/16 & IN1K & 1600 & 46.2 & 71.6 & 27.5 & 72.2 & 65.6 & 92.3 & 33.7 & 53.0 & 20.9 & 41.8 \\
            MAE \citep{mae} & ViT-B/16 & IN1K & 800 & 37.9 & 65.5 & 20.6 & 66.2 & 52.3 & 88.8 & 18.4 & 32.7 & 39.2 & 62.3 \\
            MoCo v3 \citep{mocov3} & ViT-B/16 & IN1K & 300 & 48.2 & 73.0 & 29.6 & 72.9 & 69.6 & 92.9 & 71.3 & 88.4 & 60.8 & 83.8 \\
            DINO \citep{dino} & ViT-B/16 & IN1K & 800 & 45.7 & 71.5 & 28.8 & 72.0 & 63.6 & 91.6 & 76.1 & 89.8 & 77.8 & 93.7 \\
            DINO \citep{dino} & ViT-B/8 & IN1K & 800 & 47.0 & 72.3 & 30.5 & 72.7 & 65.3 & 92.3 & 77.3 & 90.5 & 79.7 & 94.8 \\
            iBOT \citep{ibot} & ViT-B/16 & IN1K & 400 & 52.2 & 75.3 & 34.3 & 75.2 & 73.3 & \underline{93.9} & 77.0 & 90.3 & 79.2 & 94.5 \\
            Mugs \citep{mugs} & ViT-B/16 & IN1K & 400 & 51.4 & 74.1 & 34.6 & 74.5 & 72.7 & 93.5 & 78.0 & \textbf{91.3} & 80.1 & 94.5 \\
            DINOv2$^\dagger$ \citep{dinov2} & ViT-B/14 & IN1K & 400 & \underline{53.0} & \underline{75.6} & \underline{35.5} & \underline{75.8} & \underline{74.9} & \underline{93.9} & \underline{78.9} & 90.7 & \underline{81.0} & \underline{95.3} \\
            \rowcolor{ibotcyan!50}
            \textbf{CODI (ours)} & ViT-B/14 & IN1K & 400 & \textbf{55.1} & \textbf{76.6} & \textbf{36.3} & \textbf{76.2} & \textbf{77.2} & \textbf{94.6} & \textbf{79.2} & \underline{90.8} & \textbf{81.2} & \textbf{95.5} \\
            \midrule
            \multicolumn{14}{l}{\textit{Large Models}} \\
            CAE \citep{cae} & ViT-L/16 & IN1K & 1600 & 48.8 & 73.1 & 31.9 & 74.6 & 70.7 & 93.2 & 64.0 & 81.9 & 48.2 & 70.9 \\
            iBOT \citep{ibot} & ViT-L/16 & IN1K & 250 & 53.5 & 76.2 & 36.0 & 76.0 & 74.7 & 94.4 & 78.0 & 90.8 & 80.6 & 95.3 \\
            Mugs \citep{mugs} & ViT-L/16 & IN1K & 250 & 53.2 & 75.0 & 36.3 & 75.3 & 74.6 & 93.9 & 80.2 & 91.5 & 81.3 & \underline{95.4} \\
            I-JEPA \citep{ijepa} & ViT-H/14 & IN1K & 300 & 45.7 & 69.3 & 28.1 & 70.2 & 69.6 & 92.3 & 71.4 & 87.9 & 75.4 & 91.0 \\
            DINOv2$^\dagger$ \citep{dinov2} & ViT-L/14 & IN1K & 200 & \underline{56.1} & \underline{77.5} & \underline{39.1} & \underline{77.6} & \underline{78.0} & \underline{94.9} & \underline{80.4} & \textbf{92.3} & \underline{81.6} & \underline{95.4} \\
            \rowcolor{ibotcyan!50}
            \textbf{CODI (ours)} & ViT-L/14 & IN1K & 200 & \textbf{58.1} & \textbf{78.3} & \textbf{40.6} & \textbf{78.1} & \textbf{80.1} & \textbf{95.5} & \textbf{81.2} & \underline{91.7} & \textbf{82.6} & \textbf{96.2} \\
            \bottomrule
        \end{tabular}
    }
\end{table*}

\subsection{Pretraining Setting}
\label{subsec:exp_setting}
Following the standard setup \citep{dino, ibot}, we pretrain our model on ImageNet-1K \citep{imagenet} using the AdamW optimizer \citep{adamw} with a batch size of 1024. We use a ViT with a patch size of 14 as the encoder backbone and employ register tokens \citep{register}. We pretrain CODI with ViT-S/14 for 800 epochs, ViT-B/14 for 400 epochs, and ViT-L/14 for 200 epochs. The initial learning rate is set to $2\times10^{-3} \times \sqrt{\texttt{batch\_size}/1024}$ and decays with a cosine annealing scheduler. The \texttt{[CLS]}-to-patch distribution mixing ratio $\alpha$ is set to $0.1$. Detailed parameter settings are provided in Appendix~\ref{app:pretrain_setting}.

\subsection{Overall Performance}
\label{subsec:exp_overall}
In this subsection, we evaluate CODI across seven protocols. We compare with representative SSL methods from different paradigms, including 1) \textbf{Contrastive Methods:} MoCo v2 \citep{mocov2} and MoCo v3 \citep{mocov3}, 2) \textbf{Masked Modeling Methods:} MAE \citep{mae} and CAE \citep{cae}, 3) \textbf{Clustering-based Methods:} SwAV \citep{swav}, DINO \citep{dino}, iBOT \citep{ibot}, Mugs \citep{mugs}, and DINOv2 \citep{dinov2}, and 4) \textbf{Dense SSL Method:} Leopart \citep{leopart}. 

\providecolor{ibotcyan}{cmyk}{.3,0,0,0}
\begin{table*}[t]
    \centering
    \caption{Depth estimation performance on NYU-Depth V2, KITTI, and SUNRGBD zero-shot transfer. Rank is computed within each model size.}
    \label{tab:depth_estimation_main}
    \resizebox{\linewidth}{!}{
        \begin{tabular}{llc|ccc|ccc|ccc|c}
            \toprule
            Methods & Arch. & Epochs & \multicolumn{3}{c|}{NYU-Depth V2} & \multicolumn{3}{c|}{KITTI} & \multicolumn{3}{c|}{NYU$\rightarrow$SUNRGBD} & Rank \\
            \cmidrule(lr){4-6} \cmidrule(lr){7-9} \cmidrule(lr){10-12} \cmidrule(lr){13-13}
            & & & $\delta_1\uparrow$ & Rel $\downarrow$ & RMSE $\downarrow$ & $\delta_1\uparrow$ & Rel $\downarrow$ & RMSE $\downarrow$ & $\delta_1\uparrow$ & Rel $\downarrow$ & RMSE $\downarrow$ & $\downarrow$ \\
            \midrule
            \multicolumn{13}{l}{\textit{Distilled from Larger Models}} \\
            \rowcolor{gray!12}
            \textcolor{gray!65!black}{DINOv2 \citep{dinov2}} & \textcolor{gray!65!black}{ViT-S/14} & \textcolor{gray!65!black}{--} & \textcolor{gray!65!black}{84.8} & \textcolor{gray!65!black}{0.134} & \textcolor{gray!65!black}{0.437} & \textcolor{gray!65!black}{91.0} & \textcolor{gray!65!black}{0.096} & \textcolor{gray!65!black}{3.119} & \textcolor{gray!65!black}{77.4} & \textcolor{gray!65!black}{0.269} & \textcolor{gray!65!black}{0.473} & \textcolor{gray!65!black}{--} \\
            \rowcolor{gray!12}
            \textcolor{gray!65!black}{DINOv2 \citep{dinov2}} & \textcolor{gray!65!black}{ViT-B/14} & \textcolor{gray!65!black}{--} & \textcolor{gray!65!black}{87.8} & \textcolor{gray!65!black}{0.120} & \textcolor{gray!65!black}{0.408} & \textcolor{gray!65!black}{92.0} & \textcolor{gray!65!black}{0.092} & \textcolor{gray!65!black}{2.931} & \textcolor{gray!65!black}{80.7} & \textcolor{gray!65!black}{0.255} & \textcolor{gray!65!black}{0.443} & \textcolor{gray!65!black}{--} \\
            \rowcolor{gray!12}
            \textcolor{gray!65!black}{DINOv2 \citep{dinov2}} & \textcolor{gray!65!black}{ViT-L/14} & \textcolor{gray!65!black}{--} & \textcolor{gray!65!black}{89.5} & \textcolor{gray!65!black}{0.111} & \textcolor{gray!65!black}{0.390} & \textcolor{gray!65!black}{92.3} & \textcolor{gray!65!black}{0.090} & \textcolor{gray!65!black}{2.891} & \textcolor{gray!65!black}{81.4} & \textcolor{gray!65!black}{0.248} & \textcolor{gray!65!black}{0.439} & \textcolor{gray!65!black}{--} \\
            \midrule
            \multicolumn{13}{l}{\textit{Small Models}} \\
            MoCo v2 \citep{mocov2} & RN-50 & 800 & 62.9 & 0.230 & 0.719 & 82.4 & 0.130 & 4.319 & 55.5 & 0.366 & 0.749 & 7.0 \\
            SwAV \citep{swav} & RN-50 & 800 & 61.9 & 0.225 & 0.721 & 84.3 & 0.123 & 3.985 & 55.0 & 0.367 & 0.753 & 6.2 \\
            DINO \citep{dino} & RN-50 & 800 & 61.6 & 0.228 & 0.726 & 83.7 & 0.127 & 4.068 & 54.7 & 0.370 & 0.755 & 7.7 \\
            MoCo v3 \citep{mocov3} & ViT-S/16 & 300 & 69.8 & 0.192 & 0.615 & 77.2 & 0.170 & 5.007 & 50.6 & 0.411 & 0.815 & 7.7 \\
            DINO \citep{dino} & ViT-S/16 & 800 & 67.7 & 0.202 & 0.651 & 83.9 & 0.132 & 4.061 & 59.6 & 0.343 & 0.678 & 5.8 \\
            iBOT \citep{ibot} & ViT-S/16 & 800 & 72.1 & 0.183 & 0.584 & 85.3 & 0.125 & 3.907 & 63.5 & 0.330 & 0.634 & 3.3 \\
            Mugs \citep{mugs} & ViT-S/16 & 800 & 72.3 & \underline{0.179} & 0.586 & 84.1 & 0.132 & 3.963 & \underline{64.4} & \textbf{0.306} & \underline{0.613} & 3.4 \\
            DINOv2$^\dagger$ \citep{dinov2} & ViT-S/14 & 800 & \underline{72.7} & 0.184 & \underline{0.583} & \textbf{86.6} & \underline{0.112} & \textbf{3.757} & 63.2 & 0.332 & 0.631 & \underline{2.6} \\
            \rowcolor{ibotcyan!50}
            \textbf{CODI (ours)} & ViT-S/14 & 800 & \textbf{74.1} & \textbf{0.176} & \textbf{0.570} & \underline{86.5} & \textbf{0.109} & \underline{3.834} & \textbf{64.5} & \underline{0.323} & \textbf{0.610} & \textbf{1.3} \\
            \midrule
            \multicolumn{13}{l}{\textit{Base Models}} \\
            MAE \citep{mae} & ViT-B/16 & 800 & 65.3 & 0.224 & 0.673 & 84.3 & 0.123 & 4.087 & 59.1 & 0.350 & 0.688 & 8.0 \\
            MoCo v3 \citep{mocov3} & ViT-B/16 & 300 & 72.5 & 0.181 & 0.585 & 85.3 & 0.122 & 3.954 & 65.9 & 0.308 & 0.606 & 5.9 \\
            DINO \citep{dino} & ViT-B/16 & 800 & 71.7 & 0.182 & 0.593 & 86.5 & 0.117 & 3.771 & 62.8 & 0.336 & 0.644 & 6.7 \\
            DINO \citep{dino} & ViT-B/8 & 800 & 72.8 & 0.178 & 0.579 & 86.9 & 0.115 & 3.697 & 65.1 & 0.314 & 0.613 & 5.3 \\
            iBOT \citep{ibot} & ViT-B/16 & 400 & 77.7 & \textbf{0.160} & 0.523 & \textbf{88.3} & \underline{0.110} & 3.514 & \textbf{69.1} & \underline{0.299} & \textbf{0.565} & \underline{2.1} \\
            Mugs \citep{mugs} & ViT-B/16 & 400 & 74.9 & \underline{0.164} & 0.550 & 87.7 & 0.111 & 3.620 & 66.6 & \textbf{0.294} & 0.585 & 3.6 \\
            DINOv2$^\dagger$ \citep{dinov2} & ViT-B/14 & 400 & \underline{78.1} & \textbf{0.160} & \underline{0.519} & \underline{88.0} & 0.112 & \underline{3.502} & \underline{68.0} & 0.310 & 0.570 & 2.8 \\
            \rowcolor{ibotcyan!50}
            \textbf{CODI (ours)} & ViT-B/14 & 400 & \textbf{78.3} & \textbf{0.160} & \textbf{0.516} & \textbf{88.3} & \textbf{0.109} & \textbf{3.487} & 67.8 & 0.306 & \underline{0.568} & \textbf{1.7} \\
            \midrule
            \multicolumn{13}{l}{\textit{Large Models}} \\
            iBOT \citep{ibot} & ViT-L/16 & 250 & 80.4 & \underline{0.147} & 0.492 & \underline{90.5} & \textbf{0.096} & \underline{3.253} & 70.1 & 0.303 & 0.551 & 2.7 \\
            Mugs \citep{mugs} & ViT-L/16 & 250 & 76.7 & 0.158 & 0.524 & 88.4 & 0.106 & 3.485 & 68.0 & \underline{0.299} & 0.574 & 4.1 \\
            I-JEPA \citep{ijepa} & ViT-H/14 & 300 & 76.5 & 0.170 & 0.534 & 86.3 & 0.119 & 3.725 & 68.4 & 0.303 & 0.568 & 4.7 \\
            DINOv2$^\dagger$ \citep{dinov2} & ViT-L/14 & 200 & \underline{81.4} & \underline{0.147} & \underline{0.481} & 90.4 & \underline{0.097} & 3.266 & \underline{71.5} & \underline{0.299} & \underline{0.529} & \underline{2.4} \\
            \rowcolor{ibotcyan!50}
            \textbf{CODI (ours)} & ViT-L/14 & 200 & \textbf{82.0} & \textbf{0.145} & \textbf{0.472} & \textbf{90.6} & \textbf{0.096} & \textbf{3.244} & \textbf{71.7} & \textbf{0.297} & \textbf{0.528} & \textbf{1.1} \\
            \bottomrule
        \end{tabular}
    }
\end{table*}

\subsubsection{Evaluation on Dense Prediction Tasks}

\noindent\textbf{Semantic Segmentation.} To evaluate whether CODI learns well-structured representations, we use the standard off-the-shelf linear probing setting. The pretrained models are frozen, and a linear head is trained on top of the feature extractor to predict the label of each pixel. We sweep the learning rate of the linear head over a predefined set. Detailed settings are provided in Appendix~\ref{app:eval_seg}.
We evaluate model performance on the widely used COCO \citep{coco}, Pascal VOC \citep{voc}, and ADE20K \citep{ade20k} datasets. The results are shown in Tab.~\ref{tab:segmentation_main}. Among ImageNet-1K-pretrained baselines with comparable model scale, CODI achieves strong segmentation performance. Compared with reproduced DINOv2 under the aligned setting, CODI improves the average mIoU by $+2.7$, $+1.7$, and $+1.9$ for ViT-S/14, ViT-B/14, and ViT-L/14, respectively. Specifically, CODI achieves 80.1 mIoU with the ViT-Large model on Pascal VOC, approaching the original large-data DINOv2 \citep{dinov2} reference while using only ImageNet-1K pretraining in our setting.

\noindent\textbf{Depth Estimation.} Similar to the semantic segmentation setting, we freeze the encoder and train a linear head to predict depth. Following DINOv3 \citep{dinov3}, we project depth values into discrete bins and sweep the learning rate for the projection head. 
We evaluate the models on the NYU-Depth v2 \citep{nyu_depth_v2}, KITTI \citep{kitti}, and SUNRGBD datasets. Specifically, SUNRGBD is evaluated under a zero-shot transfer setting, where models are trained on NYU-Depth v2 and tested only on SUNRGBD\@. The results are shown in Tab.~\ref{tab:depth_estimation_main}. Although CODI does not obtain the best performance on all tasks, it still achieves the highest overall ranking, showing its effectiveness across settings. For ViT-S/14, CODI improves NYU $\delta_1$ by $+1.4$ over reproduced DINOv2, and CODI matches or improves reproduced DINOv2 in 24 out of 27 metric comparisons across all settings.

\noindent\textbf{Semi-supervised Video Object Segmentation.} We also evaluate CODI on the DAVIS 2017 dataset \citep{davis}. Following the standard evaluation protocol, the pixel label of the first frame is given, and labels in the following frames are propagated using representation similarity. In this setting, we also sweep several key hyperparameters and report the best results. Detailed settings are introduced in Appendix~\ref{app:eval_vos}.
As shown in Tab.~\ref{tab:vos_main}, CODI achieves $66.0$ on $\mathcal{J}\&\mathcal{F}$ with the ViT-L/14 model, surpassing all comparable (non-DINOv3) baselines. Notably, CODI outperforms the original DINOv2, which is distilled from a larger model, by a clear margin of $+1.8$ on $\mathcal{J}\&\mathcal{F}$. These results suggest that CODI features are more effective for affinity-based mask propagation.

\subsubsection{Evaluation on Image-level Tasks}

\providecolor{ibotcyan}{cmyk}{.3,0,0,0}
\begin{figure}[t]
    \centering
    \providecolor{ibotcyan}{cmyk}{.3,0,0,0}
\begin{minipage}[t]{0.33\textwidth}
    \vspace{0pt}
    \centering
    \captionof{table}{DAVIS semi-supervised VOS results.}
    \label{tab:vos_main}
    \setlength{\tabcolsep}{2pt}
    \scriptsize
    \resizebox{\linewidth}{!}{
        \begin{tabular}{llccc}
            \toprule
            Methods & Arch. & $\mathcal{J}\&\mathcal{F}$ & $\mathcal{J}_{\mathrm{mean}}$ & $\mathcal{F}_{\mathrm{mean}}$ \\
            \midrule
            \rowcolor{gray!12}
            \textcolor{gray!65!black}{DINOv3} & \textcolor{gray!65!black}{ViT-7B/16} & \textcolor{gray!65!black}{71.1} & \textcolor{gray!65!black}{--} & \textcolor{gray!65!black}{--} \\
            \midrule
            DINO & ViT-B/16 & 63.2 & 61.5 & 64.9 \\
            iBOT & ViT-B/16 & 64.3 & 63.2 & 65.5 \\
            CAE & ViT-L/16 & 65.1 & 62.8 & 67.5 \\
            iBOT & ViT-L/16 & 63.8 & 62.1 & 65.5 \\
            Mugs & ViT-L/16 & 62.1 & 60.3 & 63.0 \\
            DINOv2 & ViT-L/14 & 64.2 & 62.2 & 66.3 \\
            \midrule
            \rowcolor{ibotcyan!50}
            \textbf{CODI (ours)} & ViT-B/14 & \underline{65.7} & \textbf{63.6} & \underline{67.9} \\
            \rowcolor{ibotcyan!50}
            \textbf{CODI (ours)} & ViT-L/14 & \textbf{66.0} & \underline{63.3} & \textbf{68.7} \\
            \bottomrule
        \end{tabular}
    }
\end{minipage}

    \hfill
    \begin{minipage}[t]{0.29\textwidth}
        \vspace{0pt}
        \centering
        \captionof{table}{ImageNet-1K full fine-tuning Top-1 accuracy.}
        \label{tab:full_finetuning}
        \setlength{\tabcolsep}{2pt}
        \scriptsize
        \resizebox{\linewidth}{!}{
            \begin{tabular}{llcc}
                \toprule
                Method & Arch. & Epochs & Acc. \\
                \midrule
                MoCo v3 & ViT-S/16 & 300 & 81.4 \\
                DINO & ViT-S/16 & 800 & 82.0 \\
                iBOT & ViT-S/16 & 800 & \underline{82.3} \\
                \rowcolor{ibotcyan!50}
                \textbf{CODI (ours)} & ViT-S/14 &800 & \textbf{82.7} \\
                \midrule
                MoCo v3 & ViT-B/16 & 300 & 83.2 \\
                DINO & ViT-B/16 & 400 & 83.6 \\
                iBOT & ViT-B/16 & 400 & \underline{84.0} \\
                \rowcolor{ibotcyan!50}
                \textbf{CODI (ours)} & ViT-B/14 & 400& \textbf{84.3}\\
                \bottomrule
            \end{tabular}
        }
    \end{minipage}
    \hfill
    \begin{minipage}[t]{0.34\textwidth}
        \vspace{0pt}
        \centering
        \captionof{figure}{ImageNet-1K $k$-NN classification under occlusion.}
        \label{fig:topk_vs_masked_proportion}
        \includegraphics[width=\linewidth]{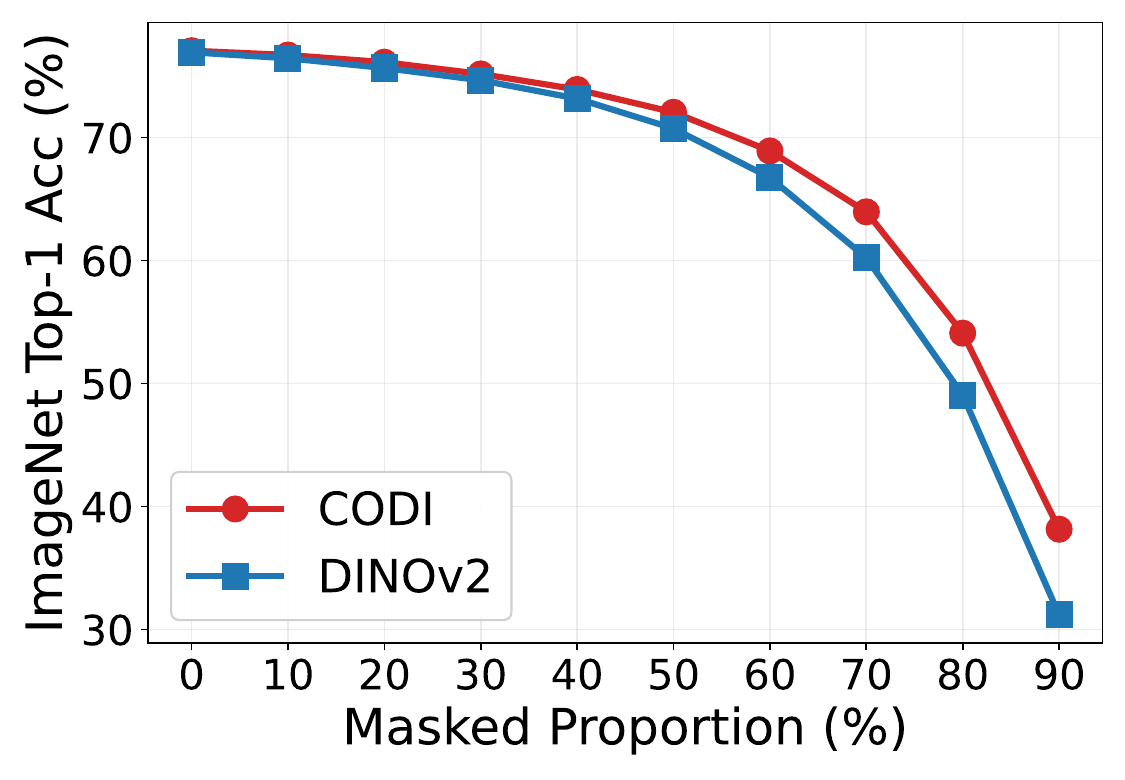}
    \end{minipage}
\end{figure}

\noindent\textbf{Image Classification with Frozen Backbone.} To provide a comprehensive analysis of the effect of CODI, we evaluate the models on image classification tasks. Following the standard protocol \citep{dino, dinov2}, we use two classifiers, a parameter-free $k$-NN classifier and a linear classifier, and report Top-1 and Top-5 accuracy. We also sweep the hyperparameters of these heads and report the best results. Details are provided in Appendix~\ref{app:eval_cls_frozen}.
As shown in Tab.~\ref{tab:classification_main}, we observe a consistent improvement in Top-1 classification performance. Notably, the Top-1 performance gain scales with model size: with ViT-Large, CODI achieves improvements of $+0.8$ and $+1.0$ in the $k$-NN and linear classification settings, respectively, compared with DINOv2.

\noindent\textbf{Full Fine-Tuning on ImageNet-1K.} Following the evaluation protocol of iBOT \cite{ibot}, we further fine-tune the model on ImageNet-1K for 100 epochs. As shown in Tab.~\ref{tab:full_finetuning}, although the performance gain is reduced after fine-tuning, CODI still improves Top-1 accuracy by $+0.4$, showing that it provides a better initialization.

\noindent\textbf{Classification with Patch Occlusion.} To study how patch-level target mixing improves image-level performance, we conduct an additional experiment on classification under occlusion. For each input image, we randomly mask a proportion of patches and use the $k$-NN classifier. The results are shown in Fig.~\ref{fig:topk_vs_masked_proportion}. Compared with DINOv2, CODI achieves better performance at every masking ratio from $0\%$ to $90\%$, and the performance gap increases as the masked ratio grows. This suggests that CODI reduces reliance on a small set of highly discriminative patches. As a result, CODI is more robust when representative patches are missing.

\subsection{Comparison with Gram Anchoring}
\label{subsec:gram_anchoring}
We compare CODI with DINOv3's Gram anchoring~\citep{dinov3} under an aligned ViT-S/14 setting. The Gram-anchoring results are obtained by fine-tuning a DINOv2 checkpoint for 100K iterations. For a fair comparison, we tune the resolution in $\{224\times224,448\times448\}$, whether to distill from the previous best or final checkpoint, and whether to update the Gram-anchor teacher, and report the best result. We also fine-tune DINOv2 with the CODI loss for the same number of iterations.
\providecolor{ibotcyan}{cmyk}{.3,0,0,0}
\begin{table}[t]
    \centering
    \caption{Comparison of ViT-S/14 pretraining and fine-tuning settings on ImageNet-1K. Gram anchoring is applied to a DINOv2 checkpoint.}
    \label{tab:gram_anchoring}
    \begin{tabular}{lrrrr}
        \toprule
        Method & Pretraining & Fine-tuning & ImageNet & VOC \\
               & iterations & iterations & $k$-NN (\%) & mIoU (\%) \\
        \midrule
        DINOv2 & 1M & -- & \underline{76.9} & 68.3 \\
        DINOv2 (best checkpoint) & 625K & -- & 75.9 & 69.4 \\
        DINOv2 + Gram anchoring & 1M & 100K & 76.7 & 70.2 \\
        DINOv2 + CODI fine-tuning & 1M & 100K & 76.8 & \underline{71.4} \\
        \rowcolor{ibotcyan!50}
        \textbf{CODI} & 1M & -- & \textbf{77.1} & \textbf{73.1} \\
        \bottomrule
    \end{tabular}
\end{table}

As shown in Tab.~\ref{tab:gram_anchoring}, although Gram anchoring recovers degraded dense performance through distillation, CODI fine-tuning outperforms it by $+1.2$ mIoU under the aligned setting. Training CODI from scratch further improves performance, surpassing the Gram-anchoring baseline by $+2.9$ mIoU with fewer training iterations. Fine-tuning does not improve classification performance, and training CODI from scratch still yields the best results. Comparisons with the distilled DINOv3 and SigLIP2 checkpoints are provided in Appendix~\ref{app:reference_models}.

\subsection{The Effect of Cross Distillation}
\label{subsec:exp_understand}
\begin{table}[t]
    \centering
    \providecolor{ibotcyan}{cmyk}{.3,0,0,0}
\begin{minipage}[t]{0.54\textwidth}
    \vspace{0pt}
    \centering
    \captionof{table}{Ablation of the CODI mixing ratio $\alpha$ and target distribution.}
    \label{tab:alpha_ablation}
    \small
    \setlength{\tabcolsep}{3pt}
        \begin{tabular}{lccccc}
            \toprule
            Target & $\alpha$ & COCO & VOC & ADE & $k$-NN \\
            \midrule
            None & 0.0 & 48.1 & 68.3 & 29.6 & 76.9 \\
            Averaged patch & 0.1 & 48.9 & 69.8 & \underline{29.9} & 76.6 \\
            \rowcolor{ibotcyan!50}
            \texttt{[CLS]} & 0.1 & \textbf{50.3} & \textbf{73.1} & \textbf{30.8} & \underline{77.1} \\
            \texttt{[CLS]} & 0.2 & \underline{49.2} & \underline{71.5} & 29.7 & \textbf{77.2} \\
            \texttt{[CLS]} & 0.5 & 42.7 & 60.5 & 25.6 & 76.1 \\
            Uniform & 0.1 & 45.8 & 62.5 & 27.3 & 75.8 \\
            \bottomrule
        \end{tabular}
\end{minipage}

    \hfill
    \providecolor{ibotcyan}{cmyk}{.3,0,0,0}
\begin{minipage}[t]{0.43\textwidth}
    \vspace{0pt}
    \centering
    \captionof{table}{Semantic alignment and regularization after 100 epochs on ImageNet-1K.}
    \label{tab:semantic_regularization}
    \small
    \setlength{\tabcolsep}{3pt}
    \begin{tabular}{lrr}
        \toprule
        Method & $k$-NN & VOC \\
        \midrule
        DINOv2 & \underline{71.1} & 61.4 \\
        Random-image \texttt{[CLS]} & 70.6 & 62.1 \\
        Entropy-matched patch & 70.8 & \underline{62.9} \\
        \rowcolor{ibotcyan!50}
        \textbf{CODI} (same-image \texttt{[CLS]}) & \textbf{71.3} & \textbf{65.2} \\
        \bottomrule
    \end{tabular}
\end{minipage}

\end{table}

\textbf{How does cross distillation affect downstream performance?} We study the target distribution and mixing strength $\alpha$, with additional dataset sweeps in Appendix~\ref{app:alpha_sensitivity}. As shown in Tab.~\ref{tab:alpha_ablation}, introducing global semantics into the patch target improves dense performance, with the \texttt{[CLS]} token outperforming the averaged patch target. Mixing a uniform distribution instead reduces both classification and segmentation performance. The results also show that $\alpha=0.1$ is an appropriate mixing strength: when $\alpha=0$, CODI reduces to DINOv2; when $\alpha$ is large, performance degrades.

On COCO, $\alpha=0.05$ and $0.075$ achieve the best $k$-NN accuracy, whereas $\alpha=0.1$ gives the best VOC mIoU. Increasing $\alpha$ to $0.2$ reduces both metrics below DINOv2. On ImageNet-22K, excluding ImageNet-1K and using a 100K-iteration schedule, $\alpha=0.1$ achieves the best results on both metrics. These findings support $\alpha=0.1$ as a strong starting point, though not a universal optimum for every downstream metric or evidence of a uniformly wide operating range. The final choice depends on whether image-level or dense performance is prioritized.

\textbf{Global semantic guidance or regularization?} To disentangle global semantic guidance from regularization, we conduct two controlled ablations after 100 epochs of ImageNet-1K pretraining. For \emph{random-image \texttt{[CLS]}}, we keep $\alpha=0.1$, but replace the teacher \texttt{[CLS]} distribution with that of another randomly selected image from the same mini-batch. This preserves a valid teacher \texttt{[CLS]} distribution while breaking image-level semantic alignment. For the \emph{entropy-matched patch target}, we first compute the entropy of the standard CODI target after mixing its patch distribution with the same-image \texttt{[CLS]} distribution at $\alpha=0.1$. We then omit \texttt{[CLS]} mixing and adjust only the patch-target temperature so that the resulting patch distribution matches this entropy. This preserves CODI's target-entropy change while keeping the target entirely local and excluding global \texttt{[CLS]} information.

As shown in Tab.~\ref{tab:semantic_regularization}, the entropy-matched patch control improves VOC mIoU by $+1.5$ over DINOv2 while reducing $k$-NN by $0.3$, showing that matching CODI's target entropy alone provides a regularization benefit. Random-image \texttt{[CLS]} yields a smaller improvement in dense prediction performance and also reduces $k$-NN. In contrast, same-image \texttt{[CLS]} (CODI) improves both metrics and yields the largest mIoU gain. Thus, regularization contributes, but semantically aligned \texttt{[CLS]}-to-patch guidance is the dominant source of improvement.

\begin{wrapfigure}{r}{0.35\textwidth}
    \vspace{-\intextsep}
    \centering
    \includegraphics[width=\linewidth]{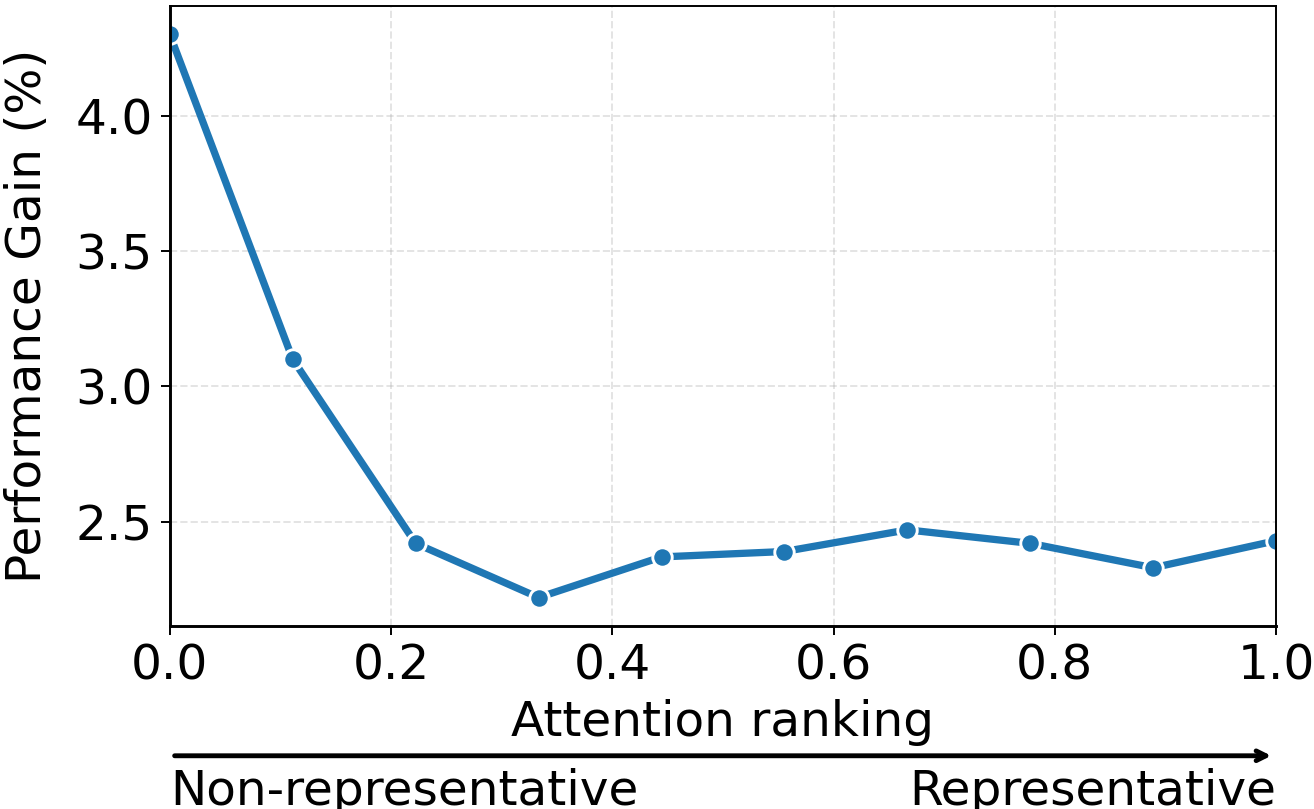}
    \caption{Performance gain by \texttt{[CLS]} attention rank.}
    \label{fig:performance_gain_ranking}
    \vspace{-\intextsep}
\end{wrapfigure}

\textbf{Does the performance gain arise from non-representative patches?} To examine whether the performance improvement comes from improving non-representative patches, we visualize the performance gain on patches with different attention rankings. We rank patches by the final-layer \texttt{[CLS]}-to-patch attention mass averaged over heads, excluding register tokens. Low-ranked patches have relatively low \texttt{[CLS]}-to-patch attention mass and thus correspond to non-representative patches. As shown in Fig.~\ref{fig:performance_gain_ranking}, the performance gain of CODI mainly comes from these non-representative patches, which is consistent with our design goal.

\providecolor{ibotcyan}{cmyk}{.3,0,0,0}
\begin{wraptable}[6]{r}{0.43\textwidth}
    \vspace{-11pt}
    \centering
    \caption{Representation distance analysis.}
    \label{tab:class_separability}
    \resizebox{\linewidth}{!}{
        \begin{tabular}{lccc}
            \toprule
            Methods & Intra-$d$ $\downarrow$ & Inter-$d$ $\uparrow$ & Gap $\uparrow$ \\
            \midrule
            DINOv2 & 0.6821 & \textbf{0.9168} & 0.2347 \\
            \rowcolor{ibotcyan!50}
            CODI & \textbf{0.5864} & 0.8430 & \textbf{0.2566} \\
            \bottomrule
        \end{tabular}
    }
\end{wraptable}

\textbf{How does CODI influence the representation structure?}
Since CODI mixes a small amount of the \texttt{[CLS]} distribution into the patch targets, a natural question arises: although this operation acts like a centering operation in the representation space, why does it improve representation separability instead of confusing representations across classes? To answer this question, we conduct an experiment. Using the pixel labels from Pascal VOC, we first aggregate patch-level labels by majority voting. Then, we compute the inter-class and intra-class distances between patch representations based on the labels. As shown in Tab.~\ref{tab:class_separability}, although both the intra-class and inter-class distances shrink as expected, the gap between them, which measures class separability, increases. These results also show a trade-off between inter-class and intra-class distances, suggesting that we need to choose $\alpha$ carefully to prevent collapse.

\section{Conclusion}
This paper studies the phenomenon that the \texttt{[CLS]} token in ViTs tends to gather information from a set of \textit{representative patches} rather than the whole object. We provide a local gradient analysis showing that, through the final \texttt{[CLS]} attention readout, image-level semantic gradients are weighted by \texttt{[CLS]}-to-patch attention mass. As a result, non-representative patches, i.e., patches with low attention scores, receive weaker image-level semantic gradients through this path, which can contribute to less discriminative patch features. To address this issue, we propose a simple yet effective cross-distillation (CODI) design. By mixing a small amount of teacher class-token information into masked patch targets, CODI ensures that every masked patch supervised by the patch loss receives a target-level global semantic correction. We empirically evaluate CODI on multiple benchmarks and evaluation settings. The results demonstrate its effectiveness and scalability.

\clearpage
\begin{ack}
This work was supported in part by National Natural Science Foundation of China: 62525212, U23B2051, 62236008, 62441232, 62521007, 62411540034, 62132006, and U21B2038, in part by Youth Innovation Promotion Association CAS, in part by the Strategic Priority Research Program of the Chinese Academy of Sciences, Grant No. XDB0680201, in part by the China National Postdoctoral Program for Innovative Talents under Grant BX20250377, in part by the China Postdoctoral Science Foundation under Grant 2026M791710, and in part by the project ZR2025ZD01 supported by Shandong Provincial Natural Science Foundation.

The authors declare no other competing interests.
\end{ack}

\bibliographystyle{plainnat}
\bibliography{neurips_2026}

\newpage

\appendix
\startcontents
\section*{Appendix Contents}
\printcontents{}{1}{\setcounter{tocdepth}{3}}

\newpage
\section{Detailed Theory and Proofs}
\label{app:theory}

This appendix provides the formal statements and proofs underlying Sec.~\ref{subsec:theory} and Sec.~\ref{sec:theory}. We deliberately keep the model local. The goal is not to describe every computational path in a multi-layer ViT, but to analyze one explicit path: the final \texttt{[CLS]} attention readout through which the image-level DINO loss aggregates patch information and sends gradients back to patch tokens.

\subsection{Why a Local \texttt{[CLS]} Readout Is Sufficient for the Argument}
\label{app:local-model}

A full ViT contains multiple self-attention blocks, residual connections, MLPs, normalization layers, and projection heads. Our main claim, however, concerns the following specific question: along the final \texttt{[CLS]} attention row, how is the gradient from the image-level teacher target distributed across patch tokens? This question can be answered exactly with a one-step local readout.

Let \(P\) denote image patch tokens in the final block, and let \(T\) denote all tokens competing in the final \texttt{[CLS]} attention row, including register tokens when used. Let \(a_{\mathrm{cls}}\) be the \texttt{[CLS]} query, and let \(k_i\) and \(v_i\) be the key and value of token \(i\in T\). The attention score, attention mass, and readout are
\begin{equation}
    s_i=\frac{a_{\mathrm{cls}}^\top k_i}{\sqrt d},
    \qquad
    \pi_i=\frac{\exp(s_i)}{\sum_{j\in T}\exp(s_j)},
    \qquad
    z=\sum_{j\in T}\pi_j v_j.
    \label{eq:app-local-readout}
\end{equation}
The output projection can be absorbed into \(v_i\), since it is a shared linear map applied after the weighted sum. During the local calculation, upstream representations, teacher targets, and all other paths are held fixed. Under this convention, for \(i\in P\), the derivatives with respect to \(v_i\) and \(s_i\) measure how the image-level loss allocates its gradient through this attention readout.

This abstraction does not imply that patch tokens receive no gradients from other network components. It only isolates a mechanism that is directly relevant to the representative-patch observation: when the \texttt{[CLS]} loss depends on a softmax-weighted readout, the gradient flowing through that readout is multiplied by the attention mass.

\subsection{Exact Gradients of the \texttt{[CLS]} Attention Readout}
\label{app:exact-gradients}

Treat the teacher class-token target \(q_{\mathrm{cls}}^t\) as fixed, and define
\begin{equation}
    \mathcal{L}_{\mathrm{cls}}
    =
    \mathrm{CE}\left(q_{\mathrm{cls}}^t,\mathrm{softmax}(\phi_s(z)/\tau_s)\right),
    \qquad
    g:=\nabla_z\mathcal{L}_{\mathrm{cls}}.
\end{equation}
For the following proof, the exact form of the head \(\phi_s\) is not important. We only need the gradient of the scalar loss with respect to the local readout.

\begin{proof}[Proof of Theorem~\ref{thm:cls-gating}]
We prove the value-path identity first. In Eq.~\eqref{eq:app-local-readout}, the readout is
\begin{equation}
    z=\sum_{j\in T}\pi_j v_j.
\end{equation}
When differentiating with respect to a value vector \(v_i\), the attention weights \(\pi_j\) are fixed because they depend on the scores \(s_j\), not on the values \(v_j\). Hence
\begin{equation}
    \frac{\partial z}{\partial v_i}=\pi_i I,
\end{equation}
where \(I\) is the identity matrix. By the chain rule,
\begin{equation}
    \nabla_{v_i}\mathcal{L}_{\mathrm{cls}}
    =
    \left(\frac{\partial z}{\partial v_i}\right)^\top
    \nabla_z\mathcal{L}_{\mathrm{cls}}
    =
    \pi_i g.
\end{equation}
This proves the first identity.

We next differentiate with respect to the score \(s_i\). The derivative of the softmax is
\begin{equation}
    \frac{\partial \pi_j}{\partial s_i}
    =
    \pi_j(\mathbf{1}_{i=j}-\pi_i).
\end{equation}
Therefore,
\begin{align}
    \frac{\partial z}{\partial s_i}
    &=
    \sum_{j\in T}
    \frac{\partial \pi_j}{\partial s_i}v_j \\
    &=
    \pi_i(1-\pi_i)v_i
    +
    \sum_{j\neq i}(-\pi_j\pi_i)v_j \\
    &=
    \pi_i v_i
    -
    \pi_i\sum_{j\in T}\pi_j v_j \\
    &=
    \pi_i(v_i-z).
\end{align}
Applying the chain rule again gives
\begin{align}
    \frac{\partial\mathcal{L}_{\mathrm{cls}}}{\partial s_i}
    &=
    \left\langle
    \nabla_z\mathcal{L}_{\mathrm{cls}},
    \frac{\partial z}{\partial s_i}
    \right\rangle \\
    &=
    \pi_i\langle g,v_i-z\rangle.
\end{align}
Defining \(u_i=-\langle g,v_i-z\rangle\), we obtain
\begin{equation}
    -\frac{\partial\mathcal{L}_{\mathrm{cls}}}{\partial s_i}
    =
    \pi_i u_i.
\end{equation}

The same calculation also yields the corresponding ratio identities. For any two patches \(i,r\in P\) with \(g\neq0\), since
\begin{equation}
    \|\nabla_{v_i}\mathcal{L}_{\mathrm{cls}}\|
    =
    \pi_i\|g\|,
\end{equation}
we have
\begin{equation}
    \frac{\|\nabla_{v_i}\mathcal{L}_{\mathrm{cls}}\|}
         {\|\nabla_{v_r}\mathcal{L}_{\mathrm{cls}}\|}
    =
    \frac{\pi_i}{\pi_r}.
\end{equation}
Similarly, when \(u_i,u_r>0\), the negative score gradients are positive and
\begin{equation}
    \frac{-\partial\mathcal{L}_{\mathrm{cls}}/\partial s_i}
         {-\partial\mathcal{L}_{\mathrm{cls}}/\partial s_r}
    =
    \frac{\pi_i u_i}{\pi_r u_r}.
\end{equation}
\end{proof}

\subsection{Attention Amplification of Representative Patches}
\label{app:attention-amplification}

The previous proof gives an instantaneous gradient statement. The next result explains how a set of useful representative patches can gain more attention under local score-gradient flow.

\begin{theorem}[Local attention amplification]
\label{thm:app-attention-amplification}
Consider the score dynamics
\begin{equation}
    \dot s_i
    =
    -\frac{\partial\mathcal{L}_{\mathrm{cls}}}{\partial s_i}
    =
    \pi_i u_i.
\end{equation}
Let \(\mathcal{R}\subseteq\mathcal{P}\) be a set of representative patches, and define its total attention mass
\begin{equation}
    \Pi_{\mathcal{R}}=\sum_{i\in\mathcal{R}}\pi_i.
\end{equation}
If
\begin{equation}
    u_i\ge \gamma>0 \quad \text{for } i\in\mathcal{R},
    \qquad
    u_j\le 0 \quad \text{for } j\notin\mathcal{R},
\end{equation}
then
\begin{equation}
    \dot\Pi_{\mathcal{R}}
    \ge
    (1-\Pi_{\mathcal{R}})\gamma
    \sum_{i\in\mathcal{R}}\pi_i^2
    \ge0.
\end{equation}
\end{theorem}

\begin{proof}
We first compute how softmax probabilities change when scores change over time. Since
\begin{equation}
    \pi_i=\frac{\exp(s_i)}{\sum_j\exp(s_j)},
\end{equation}
the time derivative is
\begin{align}
    \dot\pi_i
    &=
    \sum_j
    \frac{\partial \pi_i}{\partial s_j}\dot s_j \\
    &=
    \sum_j
    \pi_i(\mathbf{1}_{i=j}-\pi_j)\dot s_j \\
    &=
    \pi_i\dot s_i
    -
    \pi_i\sum_j\pi_j\dot s_j \\
    &=
    \pi_i\left(\dot s_i-\bar s\right),
\end{align}
where
\begin{equation}
    \bar s:=\sum_j\pi_j\dot s_j
\end{equation}
is the attention-weighted average score velocity.

Now sum over the representative set \(\mathcal{R}\):
\begin{align}
    \dot\Pi_{\mathcal{R}}
    &=
    \sum_{i\in\mathcal{R}}\dot\pi_i \\
    &=
    \sum_{i\in\mathcal{R}}\pi_i\dot s_i
    -
    \left(\sum_{i\in\mathcal{R}}\pi_i\right)
    \left(\sum_j\pi_j\dot s_j\right) \\
    &=
    \sum_{i\in\mathcal{R}}\pi_i\dot s_i
    -
    \Pi_{\mathcal{R}}
    \sum_j\pi_j\dot s_j.
\end{align}
Using \(\dot s_i=\pi_i u_i\), we obtain
\begin{align}
    \dot\Pi_{\mathcal{R}}
    &=
    \sum_{i\in\mathcal{R}}\pi_i^2u_i
    -
    \Pi_{\mathcal{R}}
    \left(
    \sum_{i\in\mathcal{R}}\pi_i^2u_i
    +
    \sum_{j\notin\mathcal{R}}\pi_j^2u_j
    \right) \\
    &=
    (1-\Pi_{\mathcal{R}})
    \sum_{i\in\mathcal{R}}\pi_i^2u_i
    -
    \Pi_{\mathcal{R}}
    \sum_{j\notin\mathcal{R}}\pi_j^2u_j.
\end{align}
By assumption, \(u_i\ge\gamma\) for \(i\in\mathcal{R}\), so
\begin{equation}
    \sum_{i\in\mathcal{R}}\pi_i^2u_i
    \ge
    \gamma\sum_{i\in\mathcal{R}}\pi_i^2.
\end{equation}
Also \(u_j\le0\) outside \(\mathcal{R}\), so
\begin{equation}
    -
    \Pi_{\mathcal{R}}
    \sum_{j\notin\mathcal{R}}\pi_j^2u_j
    \ge0.
\end{equation}
Combining these two inequalities gives
\begin{equation}
    \dot\Pi_{\mathcal{R}}
    \ge
    (1-\Pi_{\mathcal{R}})\gamma
    \sum_{i\in\mathcal{R}}\pi_i^2
    \ge0.
\end{equation}
Thus, under this local sign condition, the total attention mass on representative patches increases.
\end{proof}

\subsection{Sparse Sufficiency of an Image-Level Readout}
\label{app:sparse-sufficiency}

The attention-amplification result shows how representative patches can be reinforced. The next lemma explains why such a sparse explanation can still be sufficient for an image-level loss. Since the loss observes the pooled readout \(z\), it does not directly constrain every patch value to be used.

\begin{lemma}[Sparse sufficiency]
\label{lem:app-sparse-sufficiency}
Assume \(\ell(z)\) is \(L\)-Lipschitz on the convex hull \(\mathrm{conv}\{v_i:i\in\mathcal{P}\}\), and assume \(\|v_i\|\le B\) for all \(i\). Suppose there exists a subset \(\mathcal{R}\subseteq\mathcal{P}\) and a distribution \(\bar\pi\) supported on \(\mathcal{R}\) such that
\begin{equation}
    \ell\!\left(\sum_{i\in\mathcal{R}}\bar\pi_i v_i\right)
    \le
    \min_{\pi\in\Delta_{\mathcal{P}}}
    \ell\!\left(\sum_{i\in\mathcal{P}}\pi_i v_i\right)
    +\varepsilon.
\end{equation}
Then, for any \(\eta>0\), there exists a softmax attention distribution \(\pi^\eta\) satisfying
\begin{equation}
    \sum_{i\notin\mathcal{R}}\pi_i^\eta\le\eta
\end{equation}
and
\begin{equation}
    \ell\!\left(\sum_{i\in\mathcal{P}}\pi_i^\eta v_i\right)
    \le
    \min_{\pi\in\Delta_{\mathcal{P}}}
    \ell\!\left(\sum_{i\in\mathcal{P}}\pi_i v_i\right)
    +
    \varepsilon+2LB\eta.
\end{equation}
\end{lemma}

\begin{proof}
Let
\begin{equation}
    z_{\mathcal{R}}
    =
    \sum_{i\in\mathcal{R}}\bar\pi_i v_i.
\end{equation}
By assumption, \(z_{\mathcal{R}}\) is already within \(\varepsilon\) of the best image-level readout obtainable from all patches.

Choose a distribution \(\pi^\eta\) whose mass outside \(\mathcal{R}\) is at most \(\eta\), and whose restriction to \(\mathcal{R}\) is proportional to \(\bar\pi\). Such a distribution can be approximated by softmax scores because assigning very negative scores to patches outside \(\mathcal{R}\) makes their total softmax mass arbitrarily small.

Define
\begin{equation}
    z^\eta
    =
    \sum_{i\in\mathcal{P}}\pi_i^\eta v_i.
\end{equation}
The only difference between \(z^\eta\) and \(z_{\mathcal{R}}\) is that at most \(\eta\) probability mass has been moved from vectors in \(\mathcal{R}\) to vectors outside \(\mathcal{R}\), and then the remaining mass on \(\mathcal{R}\) has been renormalized. Since every vector has norm at most \(B\), the perturbation is bounded by
\begin{equation}
    \|z^\eta-z_{\mathcal{R}}\|
    \le
    2B\eta.
\end{equation}
Because \(\ell\) is \(L\)-Lipschitz,
\begin{equation}
    \ell(z^\eta)
    \le
    \ell(z_{\mathcal{R}})+L\|z^\eta-z_{\mathcal{R}}\|
    \le
    \ell(z_{\mathcal{R}})+2LB\eta.
\end{equation}
Using the assumed \(\varepsilon\)-optimality of \(z_{\mathcal{R}}\), we get
\begin{equation}
    \ell(z^\eta)
    \le
    \min_{\pi\in\Delta_{\mathcal{P}}}
    \ell\!\left(\sum_i\pi_i v_i\right)
    +
    \varepsilon+2LB\eta.
\end{equation}
This proves the lemma.
\end{proof}

\subsection{CODI Target Mixing}
\label{app:codi-proof}

We now prove the logit-level statement used in Sec.~\ref{sec:theory}. The result is independent of the attention readout because it concerns the masked patch prediction head directly.

\begin{proof}[Proof of Theorem~\ref{thm:codi-gradient}]
Use the notation of Theorem~\ref{thm:codi-gradient}, and write \(\mathcal{L}_i^0=\mathcal{L}_{\mathrm{patch}}^0(i)\), \(\mathcal{L}_i^\alpha=\mathcal{L}_{\mathrm{patch}}^\alpha(i)\), and \(\widetilde q_i^t=(1-\alpha)q_i^t+\alpha q_{\mathrm{cls}}^t\).
For any target distribution \(y\), define the softmax cross-entropy
\begin{equation}
    \ell(\xi;y)
    =
    -y^\top\log\mathrm{softmax}(\xi).
\end{equation}
The standard derivative is
\begin{equation}
    \nabla_{\xi}\ell(\xi;y)
    =
    \mathrm{softmax}(\xi)-y.
\end{equation}
Applying this identity to the original patch target \(q_i^t\), we obtain
\begin{equation}
    -\nabla_{\xi_i}\mathcal{L}_i^0
    =
    q_i^t-p_i^s.
\end{equation}
Applying it to the mixed target \(\widetilde q_i^t\), we get
\begin{align}
    -\nabla_{\xi_i}\mathcal{L}_i^\alpha
    &=
    \widetilde q_i^t-p_i^s \\
    &=
    (1-\alpha)q_i^t+
    \alpha q_{\mathrm{cls}}^t-p_i^s \\
    &=
    q_i^t-p_i^s+
    \alpha(q_{\mathrm{cls}}^t-q_i^t).
\end{align}
Thus, compared with the original patch target, CODI adds the term
\begin{equation}
    \alpha(q_{\mathrm{cls}}^t-q_i^t)
\end{equation}
directly at the patch logits. This term is defined by the two teacher target distributions and does not involve the \texttt{[CLS]} attention mass \(\pi_i\).
\end{proof}

\subsection{Register Tokens as Softmax Competitors}
\label{app:registers}

The main text includes register tokens as softmax competitors through \(T\). Equivalently, if \(\mathcal{G}\) is the set of register indices, the local readout becomes
\begin{equation}
    z
    =
    \sum_{i\in\mathcal{P}}\pi_i v_i
    +
    \sum_{r\in\mathcal{G}}\pi_r v_r,
    \qquad
    \sum_{i\in\mathcal{P}}\pi_i+
    \sum_{r\in\mathcal{G}}\pi_r=1.
\end{equation}
The same gradient identities apply to every token in the enlarged softmax. Let
\begin{equation}
    \Pi_{\mathcal{G}}
    =
    \sum_{r\in\mathcal{G}}\pi_r
\end{equation}
be the total attention mass assigned to registers. Then the value-path gradients satisfy
\begin{equation}
    \sum_{i\in\mathcal{P}}
    \|\nabla_{v_i}\mathcal{L}_{\mathrm{cls}}\|
    =
    (1-\Pi_{\mathcal{G}})\|g\|,
    \qquad
    \sum_{r\in\mathcal{G}}
    \|\nabla_{v_r}\mathcal{L}_{\mathrm{cls}}\|
    =
    \Pi_{\mathcal{G}}\|g\|.
\end{equation}
Registers can therefore be useful for storing global or internal computation, but in this local readout they also participate in the same softmax competition as patch tokens. In contrast, the CODI correction in Appendix~\ref{app:codi-proof} is applied to masked patch logits, so it is not reduced when the \texttt{[CLS]} readout assigns attention to registers.

\clearpage
\section{Pseudo Code}
\begin{algorithm}[h]
\caption{Self-Distillation with CODI Target Mixing}
\label{alg:codi}
\small
\begin{algorithmic}[1]
\Require Dataset $\mathcal{D}$; student $F_\theta=(f_\theta,h_\theta)$; EMA teacher $F_\xi=(f_\xi,h_\xi)$
\Require Student views $\mathcal{U}$; teacher views $\mathcal{G}$; masked patch sampler $\mathcal{M}$; mixing ratio $\alpha$

\For{each training iteration $t$}
    \State Sample images $x\sim\mathcal{D}$ and generate augmented views
    \State Sample masked patch positions $\mathcal{M}(u)$ for masked student views $u\in\mathcal{U}$
    \State Get teacher temperature $\tau_t$ and EMA momentum $m_t$

    \Statex \textbf{Teacher targets, no gradient}
    \For{each teacher view $v\in\mathcal{G}$}
        \State $(z^{T,\mathrm{cls}}_v,\{z^{T,\mathrm{p}}_{v,i}\}_{i=1}^{N})
            \gets f_\xi(v)$
        \State $q^{\mathrm{cls}}_v
            \gets \mathrm{softmax}((h_\xi(z^{T,\mathrm{cls}}_v)-c_{\mathrm{cls}})/\tau_t)$
        \For{each patch position $i$}
            \State $q^{\mathrm{p}}_{v,i}
                \gets \mathrm{softmax}((h_\xi(z^{T,\mathrm{p}}_{v,i})-c_{\mathrm{patch}})/\tau_t)$
        \EndFor
    \EndFor
    \For{each paired masked student view $u$ and teacher view $v$}
        \For{each masked patch $i\in\mathcal{M}(u)$}
            \State $\widetilde q^{\mathrm{p}}_{v,i}
                \gets (1-\alpha)q^{\mathrm{p}}_{v,i}+\alpha q^{\mathrm{cls}}_v$
        \EndFor
    \EndFor
    \State Update centers $c_{\mathrm{cls}}$ and $c_{\mathrm{patch}}$ from unmixed teacher logits

    \Statex \textbf{Student forward and losses}
    \For{each student view $u\in\mathcal{U}$}
        \State $(z^{S,\mathrm{cls}}_u,\{z^{S,\mathrm{p}}_{u,i}\}_{i=1}^{N})
            \gets f_\theta(u,\mathcal{M}(u))$
        \State $s^{\mathrm{cls}}_u\gets h_\theta(z^{S,\mathrm{cls}}_u)$
        \State $s^{\mathrm{p}}_{u,i}\gets h_\theta(z^{S,\mathrm{p}}_{u,i})$ for $i\in\mathcal{M}(u)$
    \EndFor
    \State $\mathcal{L}_{\mathrm{DINO}}
        \gets \sum_{(u,v)\in\mathcal{V}_{\mathrm{cls}}}
        \mathrm{CE}(q^{\mathrm{cls}}_v,s^{\mathrm{cls}}_u)$
    \State $\mathcal{L}_{\mathrm{CODI}}
        \gets \sum_{(u,v)\in\mathcal{V}_{\mathrm{patch}}}
        \frac{1}{|\mathcal{M}(u)|}\sum_{i\in\mathcal{M}(u)}
        \mathrm{CE}(\widetilde q^{\mathrm{p}}_{v,i},s^{\mathrm{p}}_{u,i})$
    \State $\mathcal{L}\gets\mathcal{L}_{\mathrm{DINO}}+\lambda\mathcal{L}_{\mathrm{CODI}}$

    \State Update student parameters $\theta$ using $\nabla_\theta\mathcal{L}$
    \State $\xi \gets m_t\xi+(1-m_t)\theta$
\EndFor

\Statex \textit{$\mathcal{V}_{\mathrm{cls}}$ and $\mathcal{V}_{\mathrm{patch}}$ denote the usual cross-view class-token and masked-patch distillation pairs.}
\Statex \textit{$\mathrm{CE}(q,s)=-\sum_k q_k\log\mathrm{softmax}(s/\tau_s)_k$.}
\end{algorithmic}
\end{algorithm}

\section{Detailed Settings}
\subsection{Pretraining Settings}
\label{app:pretrain_setting}
\paragraph{Pretraining protocol.}
Unless otherwise specified, we pretrain on ImageNet-1K with a DINOv2/iBOT-style teacher-student framework. The default
run uses ViT-S/14 backbones with 4 register tokens and a shared DINO/iBOT projection head. For the backbone, LayerScale
is set to $1\times10^{-5}$, and uniform stochastic depth is set to $0.1$. For the projection head, we use 65,536 prototypes,
hidden dimension 2048, bottleneck dimension 256, and 3 MLP layers. Each image is augmented into two $224\times224$
global crops sampled from scale $[0.32,1.0]$ and eight $98\times98$ local crops sampled from scale $[0.05,0.32]$,
together with the standard DINOv2/iBOT augmentations: bicubic random resized crop, horizontal flip, color jitter,
grayscale conversion, Gaussian blur, solarization, and ImageNet normalization. We train for 800 epochs on 8 GPUs with
batch size 128 per GPU, giving a global batch size of 1024 and 1250 iterations per epoch.
For compute resources, the ViT-S/14 experiments are run on 8 NVIDIA RTX 4090 GPUs, and the 800-epoch pretraining run
takes about 7 days. The ViT-B/14 and ViT-L/14 experiments are run on 8 NVIDIA A100 GPUs.

For optimization, we use AdamW with $(\beta_1,\beta_2)=(0.9,0.999)$ and cosine learning-rate decay from the sqrt-scaled
base learning rate $2\times10^{-3}$ to $1\times10^{-6}$ after 10 warmup epochs. Weight decay follows a cosine schedule
from 0.04 to 0.4. We clip gradients at 3.0, freeze the last layer for one epoch, set the patch-embedding learning-rate
multiplier to 0.2, and use layer-wise learning-rate decay 0.9. The teacher is updated with EMA, whose momentum is
cosine-scheduled from 0.992 to 1.0. The teacher temperature is warmed from 0.04 to 0.07 over 30 epochs, while the student
temperature is 0.1 and the centering momentum is 0.9.

For the training losses, DINO and iBOT are both weighted by 1.0, KoLeo regularization uses a weight of 0.1, and prototype
coding regularization is disabled, while iBOT masks 50\% of global-view samples with mask ratios in $[0.1,0.5]$. CODI keeps
the protocol unchanged except for the patch target: for each masked teacher patch, we mix the teacher class-token
distribution into the original patch-token target with a class-to-patch ratio of 0.1.

\subsection{Evaluation Settings}
\label{app:eval_setting}
\phantomsection
\label{app:eval_cls_frozen}
\noindent\textbf{ImageNet $k$-NN.}
We evaluate image-level frozen representations with weighted $k$-NN classification on ImageNet-1K. The training split is
used as the retrieval bank, and each validation image is classified by a temperature-weighted vote over its nearest
training features. To make the comparison deterministic, both bank and query images are resized on the shorter side to 256,
center-cropped to $224 \times 224$, converted to tensors, and normalized with ImageNet statistics. We use the last-block
CLS feature, apply $\ell_2$ normalization, and evaluate with global batch size 512. The standard setting keeps the first
nearest neighbor, uses no patch masking, sweeps $k \in \{10,20,100,200\}$ at temperature $\tau=0.07$, and reports the
best top-1/top-5 accuracy over the sweep.

\noindent\textbf{ImageNet linear.}
We train linear classifiers on top of frozen ImageNet-1K backbones, optimizing only the classification heads. The probe
uses the last-block CLS feature without average pooling. Training images are randomly resized and cropped to
$224 \times 224$, horizontally flipped with probability 0.5, and normalized with ImageNet statistics; validation images
are resized to 256 and center-cropped to $224 \times 224$. We train the heads for all candidate learning rates jointly
for 10 epochs, or 12,500 updates, with SGD, momentum 0.9, zero weight decay, cosine learning-rate decay, and global batch
size 1024. The base learning-rate candidates are
$\{1\times10^{-4},2\times10^{-4},5\times10^{-4},1\times10^{-3},2\times10^{-3},5\times10^{-3},1\times10^{-2},
2\times10^{-2}\}$, scaled in implementation by global batch size divided by 256. We select the head with the best final
validation top-1 accuracy and report its top-1/top-5 accuracy.

\noindent\textbf{ImageNet full fine-tuning.}
For full fine-tuning, we update the entire backbone together with a randomly initialized linear classification head. We
use an iBOT-style 100-epoch recipe with 20 warmup epochs, repeated augmentation, global training batch size 512, and
global evaluation batch size 1024. Training images are randomly resized and cropped to $224 \times 224$, horizontally
flipped with probability 0.5, processed by RandAugment with 2 operations and magnitude 9, normalized with ImageNet
statistics, and randomly erased with probability 0.25; evaluation uses resizing to 256 followed by center cropping to
$224 \times 224$. Optimization uses AdamW with learning rate $5\times10^{-4}$, minimum learning rate $1\times10^{-6}$,
weight decay 0.05, betas $(0.9,0.999)$, epsilon $1\times10^{-8}$, layer-wise learning-rate decay 0.65, and drop-path
rate 0.1. Mixup and CutMix use $\alpha=0.8$ and $\alpha=1.0$, respectively, with mixup/CutMix probability 1.0, switch
probability 0.5, and label smoothing 0.1. These hyperparameters are not tuned; we report the top-1/top-5 accuracy of the
best validation top-1 checkpoint.

\phantomsection
\label{app:eval_seg}
\noindent\textbf{Linear semantic segmentation.}
We evaluate dense transfer with a frozen-backbone linear segmentation probe on COCO, ADE20K, and Pascal VOC\@. For ViT
backbones, the probe reads normalized final-block patch tokens and predicts per-pixel labels with a $1\times1$
convolution trained by cross-entropy loss. Training images are randomly resized and cropped to $336 \times 336$ with
scale range $[0.8,1.0]$, horizontally flipped with probability 0.5, and normalized with ImageNet statistics; validation
images are resized to $448 \times 448$ and normalized. Each probe is trained for 2000 steps with global batch size 256
using Adam with betas $(0.9,0.999)$ and zero weight decay. We choose the learning rate independently for each dataset by
validation mIoU from
$\{1\times10^{-4},3\times10^{-4},8\times10^{-4},1\times10^{-3},2\times10^{-3},3\times10^{-3},1\times10^{-2}\}$, then
report mIoU and pixel accuracy.

\noindent\textbf{Linear depth estimation.}
We evaluate geometric transfer with a frozen-backbone linear depth probe. Dense patch features are passed to a depth head
with optional batch normalization followed by a $1\times1$ convolution that predicts 256 depth bins, which are converted
to metric depth by linear binning and linear normalization over bins. The head is trained with SigLoss for 38,400 steps
using global batch size 16, AdamW, weight decay $1\times10^{-4}$, betas $(0.9,0.99)$, gradient clipping at 35, and a
warmup-one-cycle schedule with 12,800 warmup iterations and final division factor 1000. We select the head learning rate
by validation RMSE from $\{1\times10^{-4},2\times10^{-4},3\times10^{-4},5\times10^{-4},1\times10^{-3}\}$.
Dataset-specific settings follow the standard crops and masks: NYUv2 uses random crop $416\times544$, max depth 10m, and
the NYU Eigen evaluation mask; KITTI uses random crop $352\times704$, max depth 80m, and the KITTI Garg evaluation mask;
SUNRGBD transfer trains on NYUv2 and evaluates on SUNRGBD with max depth 10m. We report RMSE, absolute relative error,
and $\delta_1$.

\phantomsection
\label{app:eval_vos}
\noindent\textbf{Video object segmentation.}
We evaluate video object segmentation by propagating masks over frozen dense features. The first-frame annotation
provides the only supervision, and later masks are inferred from affinities to the first frame and a queue of previous
predicted frames. Features are $\ell_2$-normalized before affinity computation; affinities are temperature-scaled,
optionally restricted by a spatial neighborhood, sparsified by top-$k$ source patches, and then used to propagate one-hot
masks. We select hyperparameters on a 15-video DAVIS train holdout sampled with seed 0, using J\&F-Mean as the selection
metric, and keep the selected setting unchanged for final evaluation. The holdout candidates are listed in
Tab.~\ref{tab:vos_hparam_candidates} following DINOv3 \cite{dinov3}. Unless otherwise stated, all runs use a frame short-side resolution of 480, nearest frame
alignment, and matrix multiplication chunk size 2048.

\begin{table}[h]
\centering
\caption{VOS hyperparameter candidates used in the DAVIS holdout search.}
\label{tab:vos_hparam_candidates}
\begin{tabular}{ccccc}
\toprule
Top-$k$ & Last frames & Neighborhood & Mask shape & Temperature \\
\midrule
5 & 7  & 12  & square & 0.2 \\
5 & 7  & 12  & circle & 0.2 \\
5 & 7  & 5   & square & 0.2 \\
5 & 7  & 24  & square & 0.2 \\
5 & 7  & $\infty$ & none & 0.2 \\
5 & 7  & 12  & square & 0.01 \\
5 & 7  & 12  & square & 0.1 \\
5 & 7  & 12  & square & 0.7 \\
5 & 4  & 12  & square & 0.2 \\
5 & 10 & 12  & square & 0.2 \\
5 & 15 & 12  & square & 0.2 \\
3 & 7  & 12  & square & 0.2 \\
10 & 7 & 12  & square & 0.2 \\
15 & 7 & 12  & square & 0.2 \\
10 & 15 & 12 & circle & 0.1 \\
10 & 15 & 24 & circle & 0.1 \\
10 & 15 & 36 & circle & 0.1 \\
10 & 15 & $\infty$ & none & 0.1 \\
\bottomrule
\end{tabular}
\end{table}

\section{Additional Experiments}
\label{app:exps}

\subsection{Unsupervised Semantic Segmentation}
\label{app:stego}
We evaluate CODI and the baseline representations under the STEGO framework~\citep{stego} on COCO-Stuff-27. As shown in Tab.~\ref{tab:stego}, CODI outperforms DINOv2 by $+1.6$ and $+1.2$ mIoU in the unsupervised and linear-probing settings, respectively. Despite using a larger patch size, CODI also outperforms STEGO with DINO ViT-S/8 by $+1.1$ mIoU in unsupervised segmentation, demonstrating that it learns a stronger dense representation structure.
\begin{table}[htbp]
    \centering
    \caption{COCO-Stuff-27 validation mIoU under the STEGO protocol.}
    \label{tab:stego}
    \begin{tabular}{lrr}
        \toprule
        Method & Unsupervised mIoU (\%) & Linear-probe mIoU (\%) \\
        \midrule
        PiCIE+H~\citep{cho2021picie} & 14.4 & 14.8 \\
        STEGO (DINO ViT-S/8) & 24.5 & 38.3 \\
        DINOv2 & 24.0 & 39.4 \\
        CODI & \textbf{25.6} & \textbf{40.6} \\
        \bottomrule
    \end{tabular}
\end{table}

\subsection{Object Detection and Instance Segmentation}
\label{app:detection}
We evaluate our models on COCO object detection and instance segmentation following the DINO and iBOT evaluation protocols~\citep{dino,ibot}. As shown in Tab.~\ref{tab:detection}, CODI achieves the best performance across all reported metrics, indicating that its pretrained representation provides a stronger initialization for object detection and instance segmentation.
\begin{table}[htbp]
    \centering
    \caption{Object detection and instance segmentation performance on COCO.}
    \label{tab:detection}
    \setlength{\tabcolsep}{4pt}
    \begin{tabular}{llrrrrrr}
        \toprule
        Method & Architecture & $\mathrm{AP}^{b}$ & $\mathrm{AP}^{b}_{50}$ & $\mathrm{AP}^{b}_{75}$ & $\mathrm{AP}^{m}$ & $\mathrm{AP}^{m}_{50}$ & $\mathrm{AP}^{m}_{75}$ \\
        \midrule
        Supervised & Swin Transformer & 48.1 & 67.1 & 52.5 & 41.7 & 64.4 & 45.0 \\
        MoBY~\citep{xie2021self} & Swin Transformer & 48.1 & 67.1 & 52.1 & 41.5 & 64.0 & 44.7 \\
        Supervised & ViT-S/16 & 46.2 & 65.9 & 49.6 & 40.1 & 62.9 & 42.8 \\
        DINO & ViT-S/16 & 48.4 & 67.6 & 52.4 & 42.0 & 64.8 & 45.1 \\
        iBOT & ViT-S/16 & 49.4 & 68.7 & 53.3 & 42.6 & 65.6 & 45.8 \\
        DINOv2 & ViT-S/14 & 49.9 & 69.0 & 53.8 & 43.0 & 66.0 & 46.4 \\
        CODI (ours) & ViT-S/14 & \textbf{50.4} & \textbf{69.4} & \textbf{54.5} & \textbf{43.6} & \textbf{66.7} & \textbf{47.2} \\
        \bottomrule
    \end{tabular}
\end{table}

\subsection{Results on Smaller Models}
\label{app:tiny}
To validate CODI on smaller models, we pretrain a ViT-T/14 model, where T denotes Tiny, on ImageNet-1K for 200 epochs. As shown in Tab.~\ref{tab:tiny}, CODI outperforms DINOv2 by $+0.3$ in ImageNet classification and $+3.2$ mIoU in VOC segmentation, demonstrating that CODI generalizes to smaller models.
\begin{table}[htbp]
    \centering
    \caption{ViT-T/14 results after 200 epochs of ImageNet-1K pretraining.}
    \label{tab:tiny}
    \begin{tabular}{lrr}
        \toprule
        Method & ImageNet $k$-NN (\%) & VOC mIoU (\%) \\
        \midrule
        DINOv2 & 63.2 & 51.6 \\
        CODI & \textbf{63.5} & \textbf{54.8} \\
        \bottomrule
    \end{tabular}
\end{table}

\subsection{Comparison with DINOv3 and SigLIP2}
\label{app:reference_models}
For the aligned Gram-anchoring comparison in Sec.~\ref{subsec:gram_anchoring}, we tune the resolution in $\{224\times224,448\times448\}$, whether to distill from the previous best or final checkpoint, and whether to update the Gram-anchor teacher, and report the best result. These results are obtained by fine-tuning a DINOv2 checkpoint for 100K iterations with DINOv3's Gram anchoring.

We evaluate the distilled DINOv3~\citep{dinov3} and SigLIP2~\citep{siglip2} checkpoints under the same frozen-feature classification and semantic segmentation protocols and compare them with CODI ViT-Large in Tab.~\ref{tab:reference_models}. These results are provided for reference, with the architecture and pretraining dataset reported for each model. CODI outperforms SigLIP2 by $+6.1$ mIoU while using substantially less data and no distillation from a larger model, highlighting the potential of our design.
\begin{table}[htbp]
    \centering
    \caption{Reference comparison with distilled DINOv3 and SigLIP2 checkpoints. Pretraining data and architectures differ across models.}
    \label{tab:reference_models}
    \setlength{\tabcolsep}{4pt}
    \begin{tabular}{lllrrr}
        \toprule
        Method & Architecture & Dataset & Epochs & ImageNet & VOC \\
               & & & & $k$-NN (\%) & mIoU (\%) \\
        \midrule
        DINOv3 (distilled) & ViT-L/16 & LVD-1689M & -- & \textbf{85.3} & \textbf{85.1} \\
        SigLIP2 (distilled) & ViT-L/16 & WebLI-10B & -- & 84.1 & 74.0 \\
        CODI & ViT-L/14 & IN1K & 200 & 81.2 & 80.1 \\
        \bottomrule
    \end{tabular}
\end{table}

We also pretrain DINOv3 from scratch on ImageNet-1K using RoPE, a patch size of 16, and DINOv3's improved schedule and LayerNorm design. After 100 epochs, DINOv3 ViT-S/16 obtains 70.9 ImageNet $k$-NN accuracy and 61.3 VOC mIoU, compared with 71.1 and 61.4 for DINOv2 ViT-S/14. We hypothesize that these DINOv3 design choices, developed for large-scale pretraining, require further hyperparameter tuning for small models trained from scratch on datasets at the scale of ImageNet-1K, and leave this adaptation to future work.

\subsection{Additional Experiments on Model Robustness}
We further evaluate the robustness of CODI under patch-level perturbations. Specifically, we measure ImageNet $k$-NN classification accuracy under three complementary stress tests that remove representative evidence, remove random visual evidence, or disrupt spatial structure.

\noindent\textbf{Robustness on Attention-based Patch Occlusion.} In this setting, we first compute the \texttt{[CLS]}-to-patch attention score for each patch. We then mask the patches with the highest attention scores and use the $k$-NN classifier to evaluate the resulting representations. This protocol tests whether the model can preserve recognition performance when the most representative patches are not visible.

\noindent\textbf{Robustness on Random Patch Occlusion.} Instead of masking high-attention patches, we also evaluate the model when randomly selected patches are masked. This protocol measures robustness to input-agnostic information loss.

\noindent\textbf{Robustness on Random Patch Shuffling.} Following \citep{naseer2021intriguing}, we shuffle the image using different grid sizes. A grid size of 1 keeps the image unchanged, while a grid size of 196 shuffles every individual patch. This protocol measures how well the model preserves classification performance when spatial structure is disrupted.

\begin{figure}[t]
    \centering
    \begin{subfigure}[t]{0.32\linewidth}
        \centering
        \includegraphics[width=\linewidth]{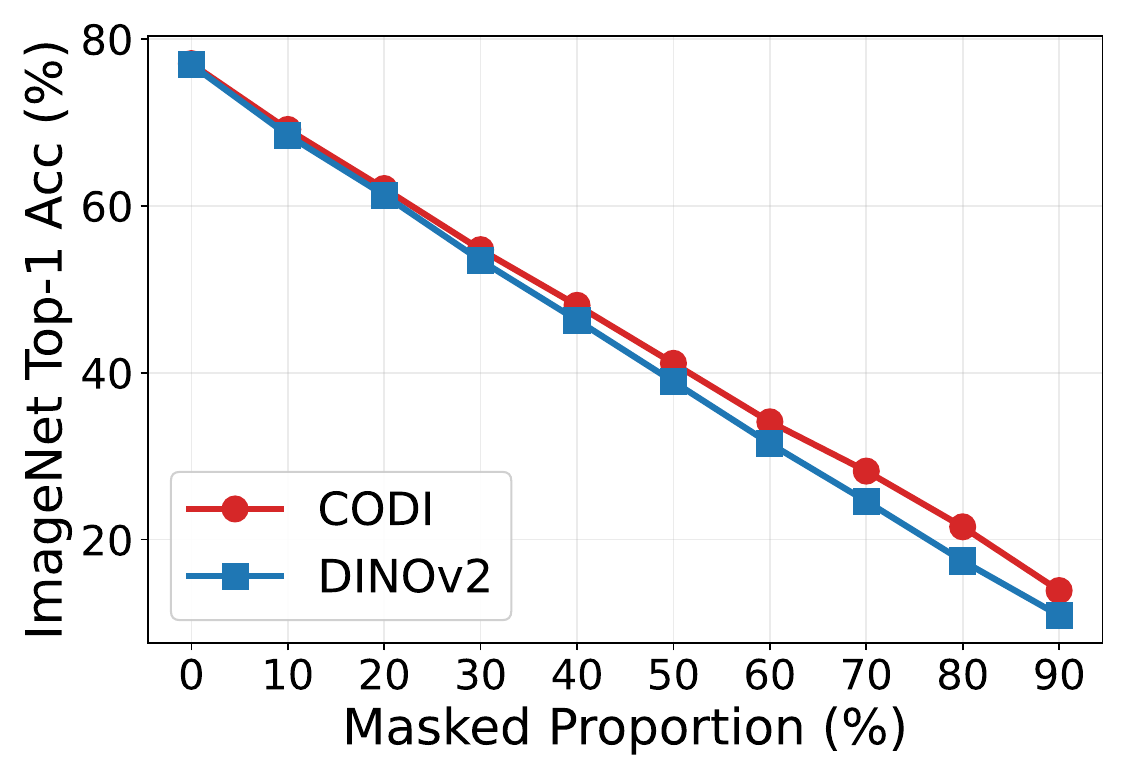}
        \caption{Attention occlusion}
        \label{fig:model_robustness_attention}
    \end{subfigure}
    \hfill
    \begin{subfigure}[t]{0.32\linewidth}
        \centering
        \includegraphics[width=\linewidth]{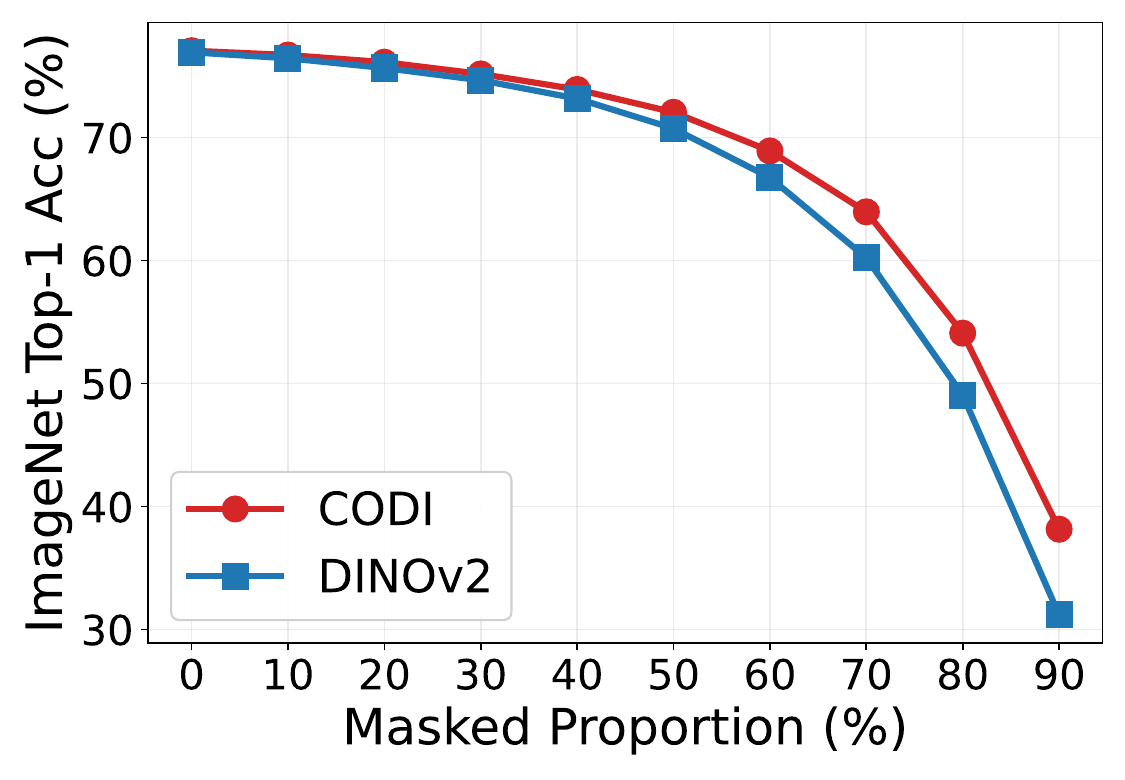}
        \caption{Random occlusion}
        \label{fig:model_robustness_random}
    \end{subfigure}
    \hfill
    \begin{subfigure}[t]{0.32\linewidth}
        \centering
        \includegraphics[width=\linewidth]{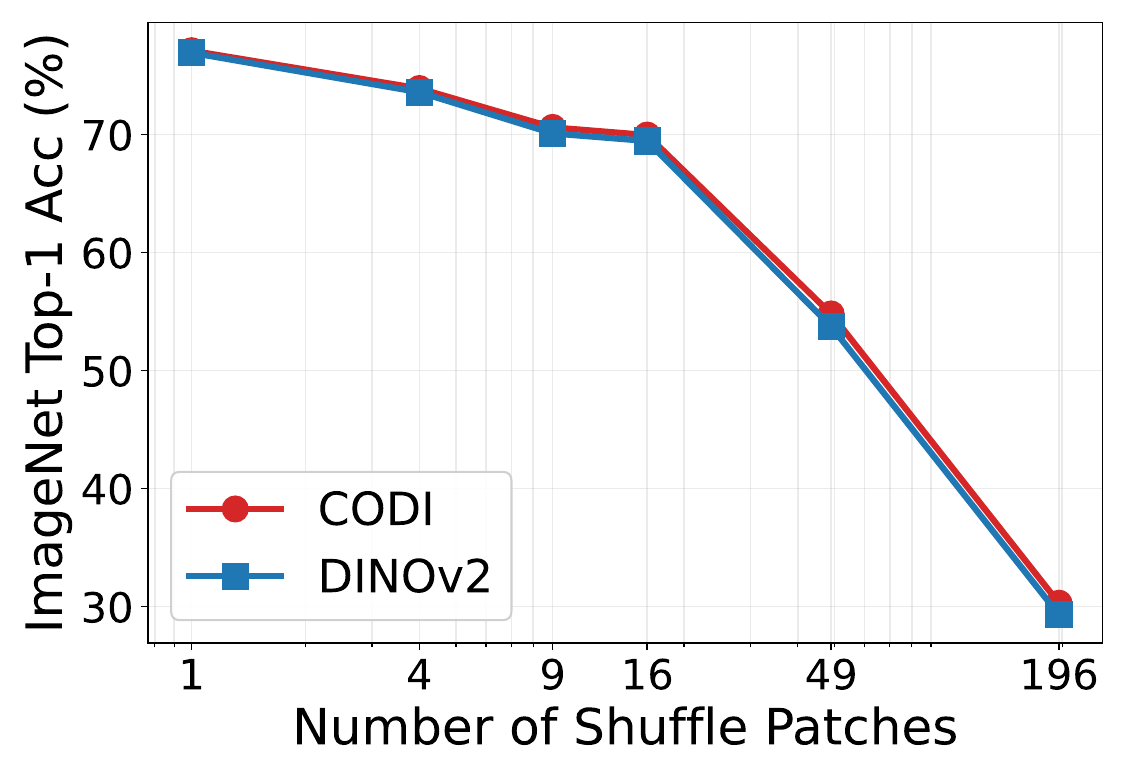}
        \caption{Patch shuffling}
        \label{fig:model_robustness_shuffle}
    \end{subfigure}
    \caption{Robustness analysis on ImageNet $k$-NN classification under patch-level perturbations.}
    \label{fig:model_robustness}
\end{figure}

As shown in Fig.~\ref{fig:model_robustness}, CODI consistently improves robustness across all three stress tests. Under attention-based occlusion, CODI retains a clear advantage when high-attention patches are removed, indicating that its classification representation relies less on a small set of representative patches. Under random occlusion, the advantage becomes larger at high masking ratios, showing that CODI better exploits distributed patch evidence when the image is severely incomplete. Under patch shuffling, CODI remains slightly but consistently stronger as spatial order is increasingly disrupted. These results suggest that CODI learns more distributed and perturbation-tolerant image representations, rather than improving only the clean-input setting.

\subsection{Sensitivity to the Mixing Ratio Across Pretraining Datasets}
\label{app:alpha_sensitivity}
We conduct additional ablations on $\alpha$ when pretraining on COCO and ImageNet-22K. For COCO, we pretrain for 800 epochs. For ImageNet-22K, we exclude the ImageNet-1K subset to avoid overlap with the downstream classification benchmark and use a 100K-iteration schedule. We evaluate ImageNet-1K $k$-NN classification and VOC semantic segmentation for both pretraining datasets.
\begin{table}[htbp]
    \centering
    \caption{Sensitivity to $\alpha$ across pretraining datasets. ImageNet-22K excludes the ImageNet-1K subset. All downstream results are percentages.}
    \label{tab:alpha_sensitivity}
    \begin{tabular}{lrrr}
        \toprule
        Method & $\alpha$ & ImageNet $k$-NN & VOC mIoU \\
        \midrule
        \multicolumn{4}{l}{\textit{COCO pretraining: 800 epochs}} \\
        DINOv2 & -- & 41.1 & 49.2 \\
        CODI & 0.05 & \textbf{41.8} & 50.8 \\
        CODI & 0.075 & \textbf{41.8} & 50.9 \\
        CODI & 0.1 & 41.6 & \textbf{51.3} \\
        CODI & 0.2 & 40.9 & 48.4 \\
        \midrule
        \multicolumn{4}{l}{\textit{ImageNet-22K pretraining: 100K iterations}} \\
        DINOv2 & -- & 63.2 & 66.1 \\
        CODI & 0.05 & 63.5 & 67.9 \\
        CODI & 0.1 & \textbf{63.8} & \textbf{68.2} \\
        CODI & 0.2 & 63.4 & 67.2 \\
        \bottomrule
    \end{tabular}
\end{table}

As shown in Tab.~\ref{tab:alpha_sensitivity}, on COCO, $\alpha=0.075$ matches the $k$-NN performance obtained with $\alpha=0.05$, while its segmentation performance lies between those obtained with $\alpha=0.05$ and $\alpha=0.1$, rather than improving both metrics. This indicates a stable range of $[0.05,0.1]$, with the final choice depending on whether image-level or dense performance is prioritized. Increasing $\alpha$ to $0.2$ reduces both metrics below the DINOv2 baseline. On ImageNet-22K, $\alpha=0.1$ achieves the best results on both metrics. These findings support $\alpha=0.1$ as a strong starting point, though not a universal optimum for every downstream metric; CODI does not have a uniformly wide operating range.

\section{Qualitative Results}
\label{app:qualitative}
We provide additional qualitative results to complement the quantitative dense-prediction evaluations. We group the visualizations by task: semantic segmentation visualizes the quality of learned patch semantics, depth estimation examines whether the learned dense features preserve spatial structure, and sparse correspondence examines local feature alignment across images.

\noindent\textbf{Semantic Segmentation Visualizations.}
Fig.~\ref{fig:app_seg_grid} provides additional frozen-backbone linear segmentation results. The baseline predictions often contain fragmented object regions or background activations. CODI produces more coherent object masks in these examples, indicating that cross distillation improves the semantic quality of patch representations.

\begin{figure}[t]
    \centering
    \includegraphics[width=\linewidth]{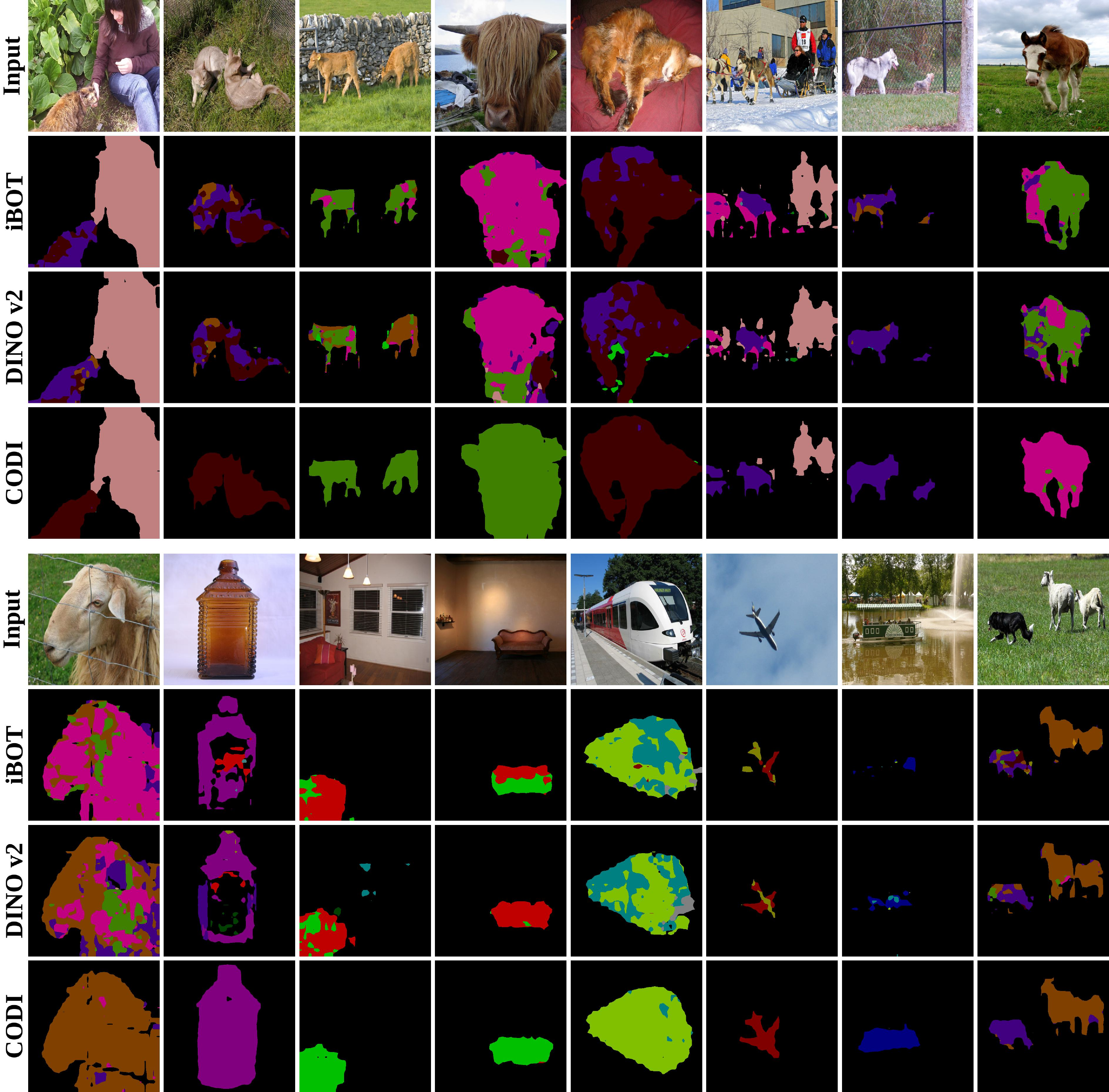}
    \caption{Additional semantic segmentation visualizations from frozen-backbone linear probes. For each block, rows from top to bottom show the input image, iBOT, DINOv2, and CODI, while columns show different examples. Colored regions denote predicted semantic labels. CODI produces more complete and less fragmented object regions in these cases.}
    \label{fig:app_seg_grid}
\end{figure}

\noindent\textbf{Depth Estimation Visualizations.}
Fig.~\ref{fig:app_depth_grid} shows additional depth predictions from frozen dense representations. Although depth estimation is less directly tied to discriminative semantics than segmentation, the visualizations show that CODI preserves coherent indoor layouts and object boundaries. This suggests that the additional class-token guidance does not sacrifice local geometric structure when improving dense semantic transfer. 

\noindent\textbf{Sparse Correspondence Visualizations.}
We further include sparse correspondence results in Fig.~\ref{fig:app_correspondence} to examine whether local features remain aligned across different images of related concepts. The matched points often connect semantically consistent object parts despite changes in instance, viewpoint, scale, and background, indicating that CODI learns patch representations that support transferable local matching.

\begin{figure}[p]
    \centering
    \includegraphics[width=\linewidth]{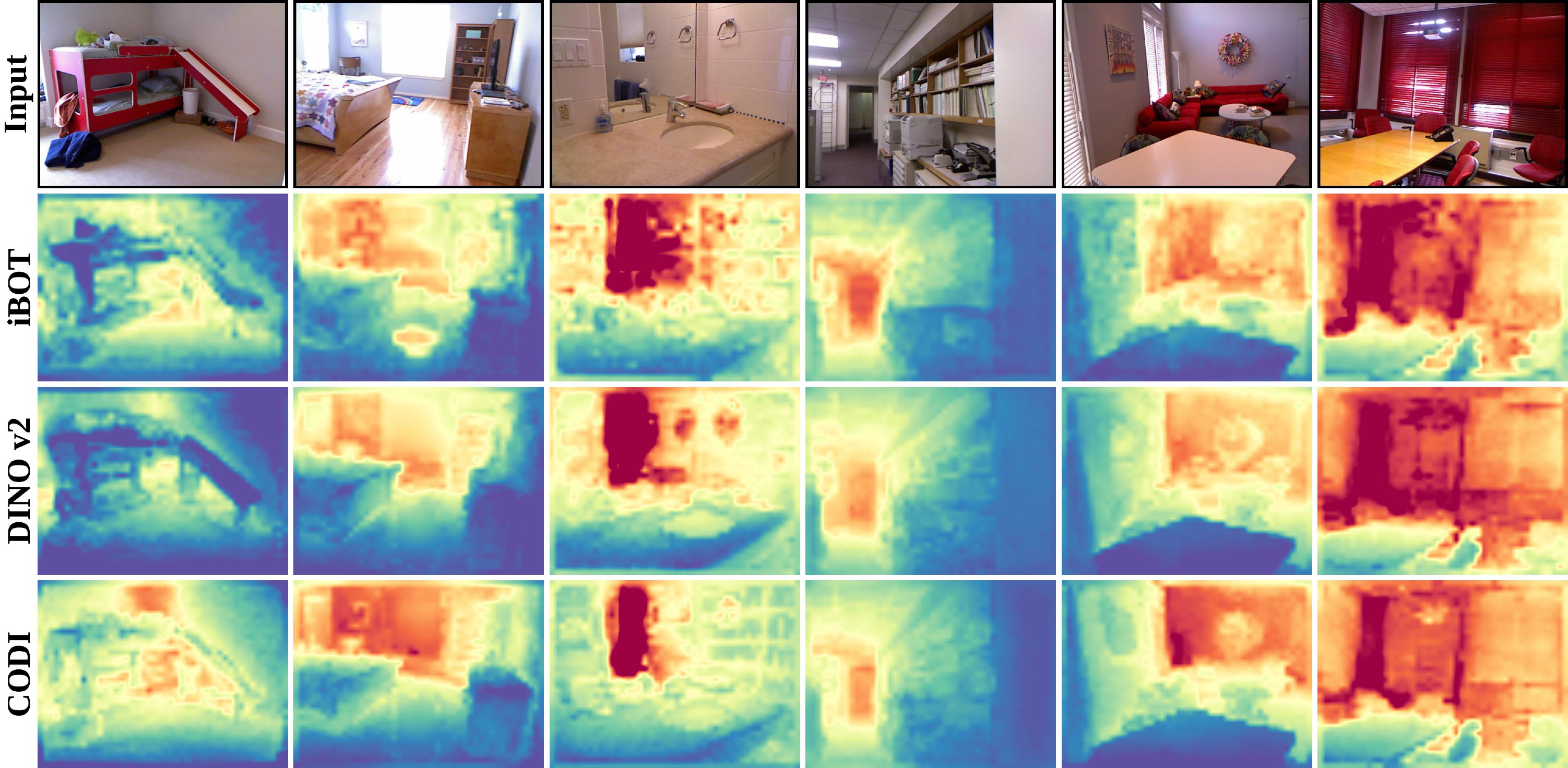}
    \caption{Additional depth estimation visualizations from frozen-backbone linear depth probes. Rows from top to bottom show the input image, iBOT, DINOv2, and CODI, while columns show different indoor examples. Colors visualize relative predicted depth. CODI maintains coherent scene layouts and depth discontinuities while improving dense representation quality.}
    \label{fig:app_depth_grid}
    
    \vspace{1.0em}
    \includegraphics[width=\linewidth]{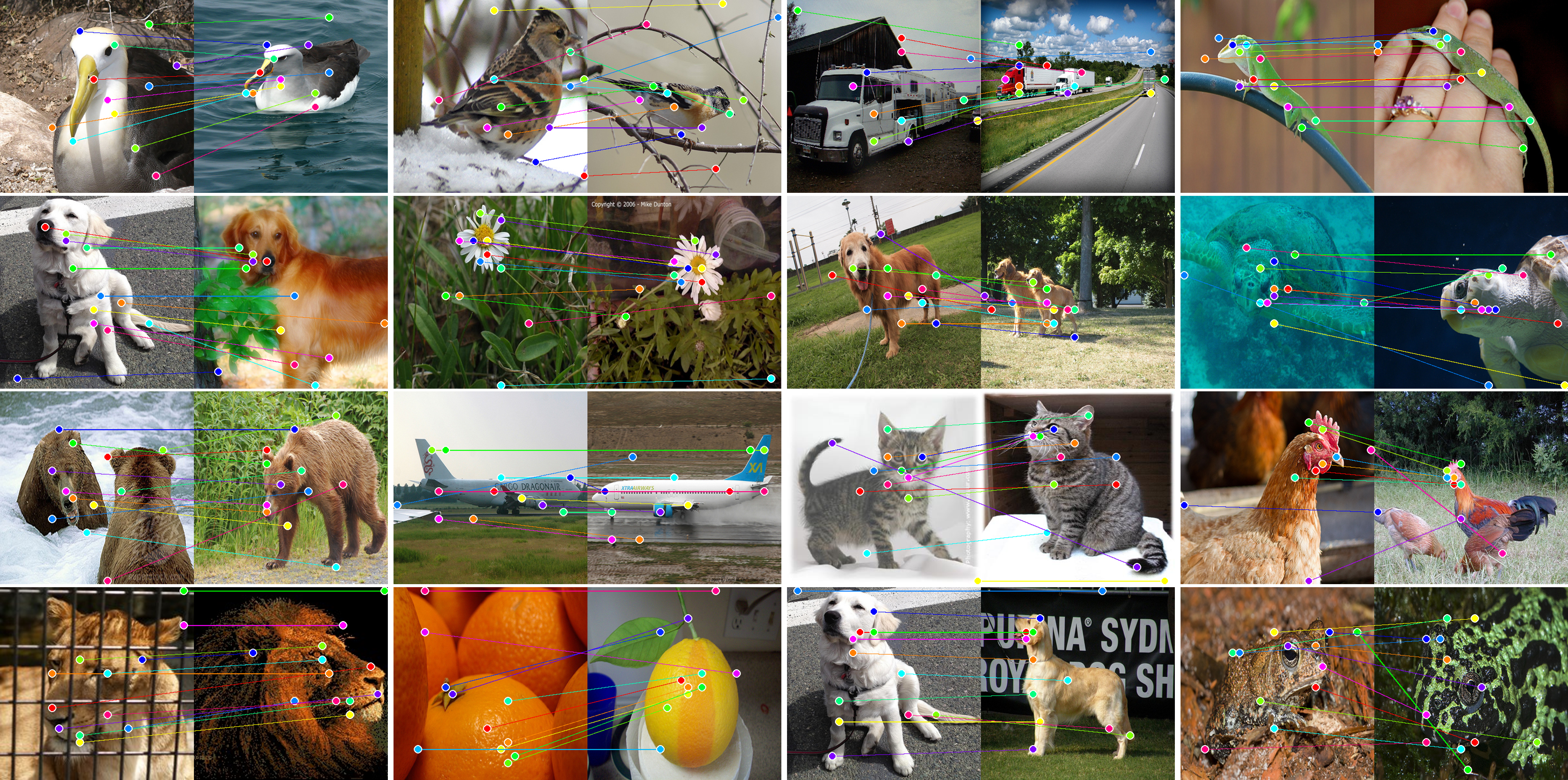}
    \caption{Additional sparse-correspondence visualizations. Colored points and lines indicate matched local features between image pairs. CODI produces many semantically aligned matches across different instances, viewpoints, scales, and backgrounds, suggesting that its dense features generalize beyond the specific image context.}
    \label{fig:app_correspondence}
\end{figure}

\section{Limitations}
\label{app:limitations}

CODI is motivated by ImageNet-style pretraining views where a dominant object often exists. In highly multi-object scenes, the \texttt{[CLS]} target can mix semantics from multiple objects, so overly strong class-token mixing may confuse patch targets. This is consistent with our ablation showing that large mixing ratios can reduce patch discriminability. Additional experiments on COCO and ImageNet-22K support \(\alpha=0.1\) as a strong starting point, but not a universal optimum or evidence of a uniformly wide operating range (Appendix~\ref{app:alpha_sensitivity}). ImageNet-22K pretraining uses only 100K iterations; full-scale training on substantially larger corpora and adaptive or multi-object-aware mixing remain future work.

\clearpage
\newpage
\section*{NeurIPS Paper Checklist}

\begin{enumerate}

\item {\bf Claims}
    \item[] Question: Do the main claims made in the abstract and introduction accurately reflect the paper's contributions and scope?
    \item[] Answer: \answerYes{}
    \item[] Justification: The abstract and Introduction state the representative-patch failure mode, the CODI target-mixing method, and the empirical scope of dense and image-level evaluations. The claims are supported by the theory in Sec.~\ref{sec:patch-starvation} and Sec.~\ref{sec:method} and by the experiments in Sec.~\ref{subsec:exp_overall}.

\item {\bf Limitations}
    \item[] Question: Does the paper discuss the limitations of the work performed by the authors?
    \item[] Answer: \answerYes{}
    \item[] Justification: The paper includes a dedicated Limitations section in Appendix~\ref{app:limitations}. It discusses the possible confusion caused by strong class-token mixing on multi-object images, sensitivity to the mixing ratio, and the limited scope of large-scale pretraining experiments.

\item {\bf Theory assumptions and proofs}
    \item[] Question: For each theoretical result, does the paper provide the full set of assumptions and a complete (and correct) proof?
    \item[] Answer: \answerYes{}
    \item[] Justification: The theoretical results and assumptions are stated in Sec.~\ref{subsec:theory} and Sec.~\ref{sec:theory}. Full formal statements and proofs are provided in Appendix~\ref{app:theory}.

    \item {\bf Experimental result reproducibility}
    \item[] Question: Does the paper fully disclose all the information needed to reproduce the main experimental results of the paper to the extent that it affects the main claims and/or conclusions of the paper (regardless of whether the code and data are provided or not)?
    \item[] Answer: \answerYes{}
    \item[] Justification: The method is specified in Sec.~\ref{sec:method} and Algorithm~\ref{alg:codi}, while pretraining and evaluation settings are described in Sec.~\ref{subsec:exp_setting} and Appendix~\ref{app:eval_setting}. These details include datasets, model variants, optimization settings, augmentation, probe protocols, and hyperparameter sweeps needed to reproduce the reported claims.
    \item[] Guidelines:
    \begin{itemize}
        \item The answer \answerNA{} means that the paper does not include experiments.
        \item If the paper includes experiments, a \answerNo{} answer to this question will not be perceived well by the reviewers: Making the paper reproducible is important, regardless of whether the code and data are provided or not.
        \item If the contribution is a dataset and\slash or model, the authors should describe the steps taken to make their results reproducible or verifiable. 
        \item Depending on the contribution, reproducibility can be accomplished in various ways. For example, if the contribution is a novel architecture, describing the architecture fully might suffice, or if the contribution is a specific model and empirical evaluation, it may be necessary to either make it possible for others to replicate the model with the same dataset, or provide access to the model. In general, releasing code and data is often one good way to accomplish this, but reproducibility can also be provided via detailed instructions for how to replicate the results, access to a hosted model (e.g., in the case of a large language model), releasing of a model checkpoint, or other means that are appropriate to the research performed.
        \item While NeurIPS does not require releasing code, the conference does require all submissions to provide some reasonable avenue for reproducibility, which may depend on the nature of the contribution. For example
        \begin{enumerate}
            \item If the contribution is primarily a new algorithm, the paper should make it clear how to reproduce that algorithm.
            \item If the contribution is primarily a new model architecture, the paper should describe the architecture clearly and fully.
            \item If the contribution is a new model (e.g., a large language model), then there should either be a way to access this model for reproducing the results or a way to reproduce the model (e.g., with an open-source dataset or instructions for how to construct the dataset).
            \item We recognize that reproducibility may be tricky in some cases, in which case authors are welcome to describe the particular way they provide for reproducibility. In the case of closed-source models, it may be that access to the model is limited in some way (e.g., to registered users), but it should be possible for other researchers to have some path to reproducing or verifying the results.
        \end{enumerate}
    \end{itemize}

\item {\bf Open access to data and code}
    \item[] Question: Does the paper provide open access to the data and code, with sufficient instructions to faithfully reproduce the main experimental results, as described in supplemental material?
    \item[] Answer: \answerYes{}
    \item[] Justification: The code and instructions for reproducing the main experiments are available at \url{https://github.com/EldercatSAM/CODI}. The experiments use public datasets, with pretraining and evaluation settings detailed in Appendix~\ref{app:pretrain_setting} and Appendix~\ref{app:eval_setting}.

\item {\bf Experimental setting/details}
    \item[] Question: Does the paper specify all the training and test details (e.g., data splits, hyperparameters, how they were chosen, type of optimizer) necessary to understand the results?
    \item[] Answer: \answerYes{}
    \item[] Justification: Sec.~\ref{subsec:exp_setting} specifies the main pretraining setup, including dataset, optimizer, batch size, model scales, epoch counts, learning-rate schedule, and CODI mixing ratio. Appendix~\ref{app:pretrain_setting} and Appendix~\ref{app:eval_setting} provide detailed pretraining and downstream evaluation protocols.

\item {\bf Experiment statistical significance}
    \item[] Question: Does the paper report error bars suitably and correctly defined or other appropriate information about the statistical significance of the experiments?
    \item[] Answer: \answerNo{}
    \item[] Justification: The paper reports single-run benchmark results and hyperparameter sweeps, but it does not report error bars, confidence intervals, or statistical significance tests. This is mainly due to the high computational cost of repeated large-scale self-supervised pretraining and downstream evaluation.

\item {\bf Experiments compute resources}
    \item[] Question: For each experiment, does the paper provide sufficient information on the computer resources (type of compute workers, memory, time of execution) needed to reproduce the experiments?
    \item[] Answer: \answerNo{}
    \item[] Justification: Appendix~\ref{app:pretrain_setting} reports GPU types and counts and the approximate wall-clock time for the ViT-S/14 800-epoch pretraining run. Memory requirements and execution times for all pretraining and downstream evaluation runs are not fully reported.
    
\item {\bf Code of ethics}
    \item[] Question: Does the research conducted in the paper conform, in every respect, with the NeurIPS Code of Ethics \url{https://neurips.cc/public/EthicsGuidelines}?
    \item[] Answer: \answerYes{}
    \item[] Justification: The work uses established public vision datasets and standard self-supervised pretraining/evaluation protocols, and it does not involve human subjects, private data collection, or deployment on users. The study is intended as foundational representation-learning research.

\item {\bf Broader impacts}
    \item[] Question: Does the paper discuss both potential positive societal impacts and negative societal impacts of the work performed?
    \item[] Answer: \answerNo{}
    \item[] Justification: The paper is foundational representation-learning work and does not include a separate broader-impact discussion.
    
\item {\bf Safeguards}
    \item[] Question: Does the paper describe safeguards that have been put in place for responsible release of data or models that have a high risk for misuse (e.g., pre-trained language models, image generators, or scraped datasets)?
    \item[] Answer: \answerNA{}
    \item[] Justification: The paper does not release high-risk assets such as language models, image generators, or scraped datasets. The proposed method is a training objective for self-supervised vision backbones evaluated on standard benchmarks.

\item {\bf Licenses for existing assets}
    \item[] Question: Are the creators or original owners of assets (e.g., code, data, models), used in the paper, properly credited and are the license and terms of use explicitly mentioned and properly respected?
    \item[] Answer: \answerNo{}
    \item[] Justification: The datasets, baselines, and model components are credited in Sec.~\ref{subsec:exp_overall} and Appendix~\ref{app:exps}. However, the paper does not explicitly enumerate the licenses and terms of use for all existing assets.

\item {\bf New assets}
    \item[] Question: Are new assets introduced in the paper well documented and is the documentation provided alongside the assets?
    \item[] Answer: \answerYes{}
    \item[] Justification: The CODI implementation and accompanying usage and reproduction documentation are provided in the project repository linked in the abstract.

\item {\bf Crowdsourcing and research with human subjects}
    \item[] Question: For crowdsourcing experiments and research with human subjects, does the paper include the full text of instructions given to participants and screenshots, if applicable, as well as details about compensation (if any)? 
    \item[] Answer: \answerNA{}
    \item[] Justification: The paper does not involve crowdsourcing, user studies, or other research with human subjects. All experiments are conducted on existing vision benchmarks.

\item {\bf Institutional review board (IRB) approvals or equivalent for research with human subjects}
    \item[] Question: Does the paper describe potential risks incurred by study participants, whether such risks were disclosed to the subjects, and whether Institutional Review Board (IRB) approvals (or an equivalent approval/review based on the requirements of your country or institution) were obtained?
    \item[] Answer: \answerNA{}
    \item[] Justification: The paper does not involve crowdsourcing, user studies, or other research with human subjects. Therefore, IRB approval or an equivalent human-subjects review is not applicable.

\item {\bf Declaration of LLM usage}
    \item[] Question: Does the paper describe the usage of LLMs if it is an important, original, or non-standard component of the core methods in this research? Note that if the LLM is used only for writing, editing, or formatting purposes and does \emph{not} impact the core methodology, scientific rigor, or originality of the research, declaration is not required.
    \item[] Answer: \answerNA{}
    \item[] Justification: LLMs were used only for writing, editing, and formatting, and are not part of the core research methodology.

\end{enumerate}

\end{document}